\documentclass[12pt]{article}

\usepackage[margin=1in]{geometry}
\usepackage{times}
\usepackage{amsmath,amssymb,amsthm}
\usepackage{mathtools}
\usepackage{bm}
\usepackage[inline]{enumitem}
\usepackage{booktabs,comment}
\usepackage{hyperref}
\usepackage[capitalise,noabbrev]{cleveref}
\usepackage{xcolor}
\usepackage{thmtools}
\usepackage{tikz}
\usetikzlibrary{arrows.meta, positioning, decorations.pathmorphing}
\usepackage{natbib}
\usepackage{enumitem}
\usepackage{microtype}
\usepackage{placeins}
\usepackage{setspace}

\newtheorem{theorem}{Theorem}[section]
\newtheorem{proposition}[theorem]{Proposition}

\newtheorem{corollary}[theorem]{Corollary}
\newtheorem{definition}[theorem]{Definition}
\newtheorem{remark}[theorem]{Remark}

\newcommand{\diam}{\operatorname{diam}}
\newcommand{\R}{\mathbb{R}}

\newcommand{\E}{\mathbb{E}}

\newcommand{\cP}{\mathcal{P}}
\newcommand{\cX}{\mathcal{X}}
\newcommand{\cZ}{\mathcal{Z}}
\newcommand{\cC}{\mathcal{C}}
\newcommand{\cH}{\mathcal{H}}

\newcommand{\cM}{\mathcal{M}}

\newcommand{\cY}{\mathcal{Y}}
\newcommand{\cL}{\mathcal{L}}

\newcommand{\width}{\mathrm{width}}

\newcommand{\trap}{\mathrm{trap}}
\newcommand{\funnel}{\mathrm{funnel}}
\newcommand{\cost}{\mathrm{cost}}

\title{The Metric Slingshot: Navigational Reuse as Width-Optimal\\
Structural Decoupling in Continual Learning}

\author{Xin Li\\
  Department of Computer Science\\
  University at Albany, SUNY\\
  Albany, NY 12222, USA\\
  \texttt{xli48@albany.edu}
}
\date{}

\begin{document}
\doublespacing
\maketitle

\begin{abstract}
The mammalian brain, most extensively studied in rodents and bats, solves an enormous variety of non-spatial cognitive tasks using neural circuitry, including grid cells, place cells, and hippocampal indexing, that originally evolved for physical navigation.
We formalize the above observation within the local Urysohn width (LUW) framework for continual learning.
The central construct is the \emph{metric slingshot}: a learned embedding $\phi: \cX \to \cZ$ that maps an arbitrary learning problem into a navigational latent space $\cZ$ where pre-evolved contraction maps (grid cells) already provide the metric machinery, so that only the topological indexing subproblem must be solved de novo.
We prove three results.
First, the optimal spacing of multi-scale grid cell modules is a geometric series whose ratio is determined by the $\Omega(w \log w)$ sample complexity bound of the LUW framework; for ecologically plausible parameters, optimality yields $r^* \approx 1.4$--$1.7$, matching electrophysiological measurements in rodent medial entorhinal cortex. Second, the slingshot preserves the width hierarchy with a Lipschitz-controlled transfer bound: a contractive embedding into a fine-resolution navigational space reduces the effective number of contexts the learner must discover.
Third, the anatomical separation of the ventral (``what'') and dorsal (``where'') visual streams achieves the structural decoupling required by Metric-Topology Factorization (MTF) \emph{by architecture}, without gradient-routing mechanisms. We demonstrate that such metric slingshot is applicable to both perception cognition and motor control. 
Together, these results provide a unified, complexity-theoretic account of navigation in non-spatial domain, grid cell multi-modularity, and hippocampal-neocortical complementary learning as consequences of a single exaptation principle: metric slingshot.

\end{abstract}

% ===========================================================================
\section{Introduction}
\label{sec:intro}
% ===========================================================================

A striking feature of mammalian cognition is that neural machinery originally evolved for spatial navigation is reused, essentially intact, for tasks that have nothing to do with physical space.
Grid cells in the medial entorhinal cortex (MEC), first characterized as encoding periodic spatial position \citep{hafting2005microstructure}, fire in grid-like patterns during abstract conceptual tasks \citep{constantinescu2016organizing}.
Hippocampal place cells, the hallmark of the cognitive map \citep{okeefe1971hippocampus}, respond to non-spatial dimensions such as sound frequency and elapsed time \citep{aronov2017mapping,pastalkova2008internally}.
Episodic memory, conceptual categorization, and even social reasoning appear to piggyback on the same hippocampal-entorhinal circuit that supports navigation \citep{epstein2017cognitive,schafer2018navigating}.
The dominant explanation, that spatial cognition is just a special case of relational cognition, restates the observation without explaining the mechanism.
What is it about navigational circuitry that makes it useful for non-spatial tasks?
And why should a system optimized for encoding Euclidean geometry in two dimensions serve as a general-purpose computational substrate?

In this paper, we provide an affirmative answer grounded in the \emph{local Urysohn width} (LUW) framework for continual learning.
The LUW framework establishes that the generalization bound for a $K$-context continual learner has a structural precondition: the number of contexts $K$ must equal or exceed the \emph{width} $w$ of the learning problem, a topological invariant that counts the minimum number of contractible basins needed to cover the problem manifold.
The associated Metric-Topology Factorization (MTF) separates inference into two subproblems: the \emph{trap} (which basin am I in?) and the \emph{funnel} (how do I reach the attractor within that basin?), with a strict requirement that these two subsystems not share gradients.
Our central construct is the \textbf{metric slingshot}: a learned embedding $\phi: \cX \to \cZ$ that maps an arbitrary learning problem into a \emph{navigational latent space} $\cZ$ where pre-evolved contraction maps (grid cells) already provide the funnel, so that only the trap must be solved from scratch.
The slingshot converts one hard problem, ``learn a contraction map for a novel task'', into three easier problems: learn an embedding ($\phi$), learn a topological index ($\Sigma$), and learn a readout ($\pi$), all operating on independent objectives and timescales.
The expensive contraction map itself is free, amortized over evolutionary time. In summary, we prove three main results:

\begin{enumerate}[leftmargin=2em]

\item \textbf{Grid cell spacing from the width scaling law} (\Cref{thm:grid-spacing}).
The optimal allocation of multi-scale grid modules, under the $\Omega(w \log w)$ sample complexity cost, is a geometric series with ratio $r^* = R^{1/(M-1)}$.
For a rodent navigating scales from body-length to territory diameter, the scaling law predicts $r^* \approx 1.4$--$1.7$, matching electrophysiological observations \citep{stensola2012entorhinal,barry2007experience}.
To our knowledge, the first derivation of grid cell multi-modularity from a complexity-theoretic principle rather than from spatial coding efficiency \citep{wei2015principle,mathis2012optimal}.

\item \textbf{Width transfer under embedding} (\Cref{prop:width-transfer}).
The metric slingshot preserves the width hierarchy with Lipschitz-controlled bounds: a contractive embedding into a fine-resolution navigational space can reduce the effective width and the sample complexity of the embedded problem. We demonstrate the generalization property of the slingshot principle is consistent with exaptation in evolution \cite{gould1982exaptation}.
 
\item \textbf{Structural decoupling by architecture} (\Cref{prop:decoupling}).
The anatomical separation of the ventral and dorsal streams, converging in the hippocampal formation, achieves the MTF decoupling requirement \emph{by construction}, without orthogonal projectors or stop-gradient operations.
The indexer $\Sigma$, the embedding $\phi$, and the readout $\pi$ are trained by separate objectives, eliminating the circular dependency between context identification and metric learning that plagues monolithic continual learners.
\end{enumerate}

The payoff is a unified account: grid cell multi-modularity, dentate gyrus sparsity \citep{jung1993comparison}, two-timescale complementary learning \citep{mcclelland1995complementary}, and coarse-to-fine hippocampal replay \citep{diba2007forward} all emerge as consequences of a single information-theoretic quantity: the structural term in the continual learning generalization bound.

The paper is organized as follows.
\Cref{sec:background} reviews the LUW framework and MTF in sufficient detail to make the paper self-contained.
\Cref{sec:evolution} argues that navigation is the canonical width problem and derives the evolutionary ordering of the slingshot components.
\Cref{sec:grid} derives the optimal multi-scale grid structure from the width scaling law - the most direct technical consequence of the navigational slingshot.
\Cref{sec:slingshot} defines the metric slingshot, proves the width transfer bound, and analyzes the learning decomposition.
\Cref{sec:decoupling} shows that the ventral-dorsal anatomical separation implements MTF decoupling by architecture.
\Cref{sec:motor} extends the what-to-where framework from perception cognition to motor control, exemplifying the second application of metric slingshot.
%\Cref{sec:predictions} states the empirical predictions of the theory.
%\Cref{sec:related} reviews related work.
\Cref{sec:conclusion} discusses limitations and open questions.

% ===========================================================================
\section{Background: Local Urysohn Width and Metric-Topology Factorization}
\label{sec:background}
% ===========================================================================

%We summarize the elements of the LUW framework needed for this paper; for a full treatment, see \citet{li2026luw}.

\paragraph{Width of a Learning Problem}
We introduce \emph{local Urysohn width} (LUW), a complexity measure for learning problems on metric spaces.
The adjective \emph{local} refers to the fact that width is defined by covers of locally contractible, diameter-bounded cells (note that it measures the minimum number of patches with well-behaved local geometry, not any global topological property of $\cX$).
Unlike VC dimension, which characterizes the richness of a hypothesis \emph{class} $\cH$ (e.g., the number of dichotomies an $n$-point set can shatter \cite{vapnik2013nature}), width characterizes the topological-geometric complexity of the \emph{problem} $\cP$ itself.
The two measures are incomparable: a problem can have large width with bounded VC dimension, or unit width with unbounded VC dimension (we have rigorously proved the VC-width separation result in \citet{li2026luw}).
More importantly, the topology of $\cX$ alone does not determine $w$: even when $\cX \simeq \R^n$ (contractible, trivially $\beta_k = 0$ for all $k \geq 1$), the \emph{target map} $f: \cX \to \cY$ can impose multiple structurally incompatible regimes - e.g., a parity function or a multi-modal density, that require $w > 1$ context cells, because no single $\gamma$-contractive map can faithfully approximate $f$ on all of $\cX$ simultaneously.

\begin{definition}[Local Urysohn width]
\label{def:width}
Let $\cP: \cX \to \cY$ be a learning problem defined on a compact metric space $(\cX, d)$, where $\cY$ is the output space equipped with a metric $d_\cY$.
Fix a contractivity constant $\gamma \in (0,1)$.
The \emph{local Urysohn width} of $\cP$ is the minimum integer $w$ such that there exists an open cover $\cX = \bigcup_{c=1}^{w} \cX_c$ where, for each $c$, the restricted problem $\cP|_{\cX_c}$ is solvable by a $\gamma$-contractive map $G_c: \cX_c \to \cY$ satisfying $d_\cY(G_c(x), G_c(x')) \leq \gamma\, d(x, x')$ for all $x, x' \in \cX_c$, and $G_c$ achieves risk at most $\varepsilon$ on $\cX_c$ for a prescribed tolerance $\varepsilon > 0$.
We call each $\cX_c$ a \emph{context cell}.
\end{definition}

\begin{remark}[Contractivity vs.\ fixed points]
\label{rem:contraction}
A $\gamma$-contractive map $G_c: \cX_c \to \cY$ on a complete metric space has a unique fixed point in $\cY$ by the Banach fixed-point theorem when $\cY = \cX_c$.
In Definition~\ref{def:width} the codomain is the output space $\cY$, so the contractivity condition is a Lipschitz-smoothness requirement on $G_c$ over $\cX_c$, ensuring that nearby inputs produce nearby outputs.
In contrast, the \emph{contraction diameter} $D_0$ used in Definition~\ref{def:nav-space} for navigational cells is the diameter of the \emph{image} $G_c^0(\cZ_c)$ after one application of the pre-built map; this measures the precision to which the contraction resolves position rather than asymptotic convergence, and is nonzero in a finite-precision system or when the maps are stopped before full convergence.
These two notions are consistent: a $\gamma$-contraction on a cell of diameter $D$ produces images of diameter at most $\gamma D$, so $D_0 \leq \gamma\, \diam(\cZ_c)$.
\end{remark}

\noindent
Intuitively, $w$ counts the minimum number of locally smooth regimes needed to cover the problem. Width is therefore a property of the problem pair $(\cX, f)$, not of $\cX$ alone.
The Betti numbers of $\cX$ provide lower bounds on $w$ when the target map $f$ is topologically non-trivial: $w \geq \beta_1(\cX) + 1$ when $f$ is incompatible with a single contractive chart.
The width is preserved under homeomorphisms of $(\cX, f)$ jointly and can only decrease under continuous surjections that flatten the problem geometry.
Note that the openness of the cover cells follows from the $\gamma$-contractivity condition together with compactness; it does not impose an \emph{a priori} fixed margin, but rather requires that the smooth-approximation error $\varepsilon$ be achievable within each cell.

\paragraph{Metric-Topology Factorization}
The LUW framework decomposes any retrieval or inference process into two subproblems: \emph{trap} and \emph{funnel}, which we formally define as follows:

\begin{definition}[Trap and Funnel]
\label{def:trap-funnel}
Consider a learning problem $\mathcal{P}$ of width $w$
with a covering $\{\mathcal{X}_c\}_{c=1}^{w}$ of the input space.
1) The \emph{trap} subproblem is to determine, for a given input
$x \in \mathcal{X}$, which basin $c$ contains $x$.
This is a \emph{topological} problem: it requires discovering
the basin structure from data;
2) The \emph{funnel} subproblem is, given the correct basin
assignment $c$, to reach the attractor within $\mathcal{X}_c$
and produce the desired output.
This is a \emph{metric} problem: it requires learning a
contractive map $G_c$ within the basin.
\end{definition}

\noindent
The \emph{Metric-Topology Factorization} (MTF) principle states that
these two subproblems should be optimized independently.
Let $\theta_\Sigma$ denote the parameters of the trap system
(context indexing) and $\theta_G$ the parameters of the funnel system
(metric contraction).
MTF requires that their gradient flows remain decoupled:
$\nabla_{\theta_\Sigma}\mathcal{L}_{\text{funnel}} = 0,
\quad
\nabla_{\theta_G}\mathcal{L}_{\text{trap}} = 0 $.
When the decoupling condition is violated,
the circular dependency between context identification
and metric learning induces catastrophic interference
\citep{mccloskey_catastrophic_1989}.
With the structural decoupling assumption, inference under MTF proceeds in two stages:
1) \textit{Index:} determine the active context
$\Sigma(x) \in \{1,\dots,w\}$;
2) \textit{Collapse:} perform contractive inference
within the selected basin.
Formally, we denote
$\mathcal{F}(x)
=
\mathcal{G}_{\Sigma(x)}\bigl(\Phi_\theta(x)\bigr)$,
where $\Phi_\theta$ denotes a shared feature representation,
$\Sigma(x)$ is the context-indexing function,
and $\mathcal{G}_{\Sigma}$ denotes the
context-specific contraction dynamics.
By separating context identification from metric shaping,
the system avoids forcing incompatible regimes
into a single global geometry \cite{mcclelland_why_1995}.
Stability is maintained through discrete indexing,
while plasticity is preserved through local adaptation
within each basin \cite{abraham2005memory}.

\paragraph{The Width Scaling Law}
The fundamental result of the LUW framework is a sample complexity lower bound for continual learners:

\begin{theorem}[Sample complexity lower bound]
\label{thm:sample-lower}
Let a continual learning problem have width \(w\), meaning that its input
distribution decomposes into \(w\) distinguishable contexts of equal mass.
Suppose correct prediction on each context requires that the context be
explicitly identified from data.
Then any learner achieving uniformly small generalization error on the
trap subproblem must use at least
$n=\Omega(w\log w)$
samples.
The \(\log w\) factor is the coupon-collector cost of observing all
\(w\) contexts sufficiently often.
\end{theorem}
\noindent
Note that Theorem \ref{thm:sample-lower} is a compact version (a complete version and its proof can be found in the Appendix). The $w \log w$ cost is strictly superlinear in the number of basins, and it governs the dominant term in the continual learning generalization bound \cite{wang_comprehensive_2024}.

% ===========================================================================
\section{Navigation as the Canonical Width Problem}
\label{sec:evolution}
% ===========================================================================

\paragraph{Navigation Has Externally Given Topology}
For most learning problems, the topology of the problem manifold is unknown \emph{a priori} and must be inferred from data.
This is the trap subproblem, which requires $\Omega(w \log w)$ samples (\Cref{thm:sample-lower}).
Navigation is special \cite{mcnaughton1991dead,mcnaughton2006path}: the topology is \emph{imposed by the environment} (e.g., walls, barriers, boundaries) and is directly experienceable through movement. This is in sharp contrast to, say, learning the topology of a classification problem, where the learner must observe many $(x, y)$ pairs to discover that the target function has multiple modes. We formalize our intuition into the following proposition.

\begin{proposition}[Navigation reduces the additional cost of trap identification]
\label{prop:nav-trap}
Consider a navigational problem on a physical environment $\cM$ in which an agent has access to a sensorimotor interface supporting localization and loop-closure detection.
Assume that representative obstacle-induced loops in $\cM$ can be explored and distinguished through this interface.
Then the topological structure relevant to the trap subproblem is supplied directly by embodied exploration rather than inferred from task labels or function values.
Consequently, the additional statistical cost of trap identification \emph{beyond physical exploration} is $O(w)$ rather than $\Omega(w\log w)$: once every context has been visited at least once through exploration, no further coupon-collection overhead is required to confirm that the full context library has been seen.
\end{proposition}
\noindent
\emph{Proof sketch.}
The $\Omega(w \log w)$ lower bound in Theorem~\ref{thm:sample-lower} arises because in a generic learning problem the learner cannot tell from any finite sample whether all $w$ contexts have been observed (the coupon-collector problem over hidden function classes).
In navigation, the topology is \emph{observable via trajectory}: a loop-closure event (returning to a previously visited location by a topologically distinct route) directly certifies that the learner has traversed a new topological handle.
After one full exploratory tour of $\cM$ that visits all $w$ obstacle-separated regions, every context has been confirmed at least once.
By Assumption 2 of Theorem~\ref{thm:sample-lower}, the additional cost of confirming each new context beyond exploration is $O(1)$ per context, giving a total trap-identification overhead of $O(w)$.
The $O(w)$ bound is tight in general: one pass through each of the $w$ regions is necessary and sufficient to certify coverage.
(A full proof is given in the Appendix.)\qed

\paragraph{Navigation Has a Stable, Euclidean Metric}
The metric of physical space is Euclidean (to biological precision) and does not change over time, which means the funnel subproblem for navigation has two special properties:
1) The contraction maps $G_c^0$ need only be learned \emph{once} (during development or evolution) and remain valid indefinitely;
2) The metric is \emph{isotropic}: the contraction diameter $D_0$ is the same in all directions, so the tiling can be regular (hexagonal, as in grid cells \citep{hafting2005microstructure}).
For non-navigational problems, the metric may be anisotropic, non-stationary, or task-dependent.
The slingshot tackles it with exaptation \cite{gould1982exaptation}: $\phi$ absorbs the task-specific metric distortion, mapping it into the isotropic navigational space where the pre-built tiling works.
%\paragraph{The Evolutionary Sequence}
These properties suggest a specific ordering of evolutionary innovations \cite{bennett2023brief}:

\begin{enumerate}
\item \textbf{Stage 1: Build the navigational funnel.}
Evolution constructs grid cells and place cells - a regular tiling of $\cZ = \R^d_{\text{nav}}$ (with $d = 2$ for surface navigation) at multiple scales.
This is a one-time investment amortized over all future organisms.

\item \textbf{Stage 2: Build the navigational trap.}
The hippocampus (dentate gyrus, CA3, CA1) evolves to solve the trap problem for spatial environments: which basin am I in?
Dentate gyrus sparsity \citep{jung1993comparison}, CA3 attractor dynamics \citep{leutgeb2007pattern}, and CA1 mismatch detection \citep{fyhn2002spatial} are architectural solutions to the trap, calibrated to navigational demands.

\item \textbf{Stage 3: Generalize via the slingshot.}
Once the navigational machinery exists, evolution discovers that \emph{non-spatial} problems can be solved by embedding them into $\cZ$ via $\phi$.
The ventral ``what'' stream learns $\phi$; the dorsal ``where'' stream provides $G_c^0$ and the readout $\pi$; the hippocampus provides $\Sigma$.
No new architectural components are needed, only a new embedding.

\item \textbf{Stage 4: Non-spatial problems piggyback.}
Episodic memory, conceptual categorization, and abstract reasoning reuse the same hippocampal-entorhinal machinery \citep{aronov2017mapping}.
The observation that place cells respond to non-spatial variables and that grid-like signals appear during abstract tasks \citep{constantinescu2016organizing} is a direct prediction: all high-level cognition tasks are slingshot embeddings of non-spatial problems into $\cZ$.
\end{enumerate}

% ===========================================================================
\section{Grid Cell Multi-Modularity from the Width Scaling Law}
\label{sec:grid}
% ===========================================================================

\paragraph{Setup: Multi-Scale Grid Modules}
We now derive the optimal multi-scale structure of grid cell modules from the width theory.
This is, to our knowledge, the first derivation of grid cell multi-modularity from a complexity-theoretic principle rather than spatial coding efficiency \citep{wei2015principle,mathis2012optimal,stemmler2015connecting}.
Entorhinal grid cells are organized into $M$ discrete modules, each tiling $\cZ$ with a characteristic spacing $s_m$ \citep{stensola2012entorhinal}.
Within module $m$, the grid cells provide a contraction map with diameter $D_0^{(m)} = c_D \cdot s_m$ for some constant $c_D \in (0, 1)$ (the ``field width'' fraction of the spacing).
The modules are ordered $s_1 < s_2 < \cdots < s_M$, and the observed ratio between successive modules is approximately $r = s_{m+1}/s_m \approx 1.4$--$1.7$ \citep{barry2007experience,stensola2012entorhinal}.

\begin{definition}[Multi-scale navigational space]
\label{def:multiscale}
A \emph{multi-scale navigational space} is a tuple $(\cZ, d_\cZ, \{s_m\}_{m=1}^M, c_D)$ where each module $m$ provides a tiling of $\cZ$ with contraction diameter $D_0^{(m)} = c_D \cdot s_m$.
The module at scale $m$ yields width $w^{(m)}(L) = \lceil L / D_0^{(m)} \rceil$ for a problem with path length $L$ in $\cZ$.
\end{definition}

\paragraph{The Width-Optimal Covering Problem}

An organism encounters problems with embedded path lengths $L$ drawn from a distribution $p(L)$ over $[L_{\min}, L_{\max}]$.
For each problem, it must choose a module $m$ to use as the primary contraction map.
A finer module (smaller $s_m$, smaller $D_0^{(m)}$) gives more contexts but each is better conditioned.
A coarser module gives fewer contexts but each may be insufficient.
The organism's strategy: use the \emph{coarsest} module $m^*(L)$ such that $D_0^{(m^*(L))}$ still satisfies the contraction condition.
The requirement $s_m \leq L$ ensures that the module's contraction diameter $D_0^{(m)} = c_D s_m$ is at most $c_D L < L$, so the path of length $L$ spans at least one full cell and non-trivial multi-context width ($w \geq 2$) is achievable.
When $s_m > L$ the entire path fits inside a single grid cell and the effective width is $w = 1$, which is trivially contractible but provides no multi-scale resolution---the module is too coarse to be useful.
Note that the grid structure would still technically exist for $s_m > L$; the condition is a selection rule that restricts attention to modules that contribute meaningful structural resolution for the problem at hand.
Formally:
\begin{equation}
\label{eq:module-selection}
m^*(L) = \max\{m : s_m \leq L\}\,.
\end{equation}
Using module $m^*(L)$, the resulting width is:
$w(L, m) = \left\lceil \frac{L}{c_D \cdot s_m} \right\rceil\,$.

\begin{definition}[Width-optimal module allocation]
\label{def:optimal-allocation}
Given a budget of $M$ modules and a distribution $p(L)$, the \emph{width-optimal allocation} minimizes the expected total cost:
\begin{equation}
\label{eq:cost}
\cost(\{s_m\}_{m=1}^M) = \E_{L \sim p}\!\left[\min_{m: s_m \leq L} w(L, m) \cdot \log w(L, m)\right],
\end{equation}
where $w(L, m) \cdot \log w(L, m)$ is the $\Omega(w \log w)$ sample complexity of the trap subproblem at width $w$.
\end{definition}
\noindent
The $w \log w$ cost is the critical ingredient.
If the cost were linear in $w$ (as a covering-number argument without the coupon-collector factor would give \cite{flajolet1992birthday}), the optimal allocation would be different.

\paragraph{Optimal Spacing: Geometric Series}
Within the MTF framework, each grid module provides a scaffold scale at which a path of length L can be amortized into a width $w=L/(c_D s_m)$. The natural design question is how the spacings $\{s_m\}$ should be chosen to minimize expected amortization cost across environments. The next theorem shows that a log-uniform ecology of path scales leads uniquely to geometric spacing \cite{stemmler2015connecting}.

\begin{theorem}[Optimal grid module spacing]
\label{thm:grid-spacing}
Let $p(L)$ be log-uniform on $[L_{\min}, L_{\max}]$ with $R = L_{\max}/L_{\min}$.
Then the cost in Eq.~\eqref{eq:cost} is minimized by a geometric series of spacings:
$s_m = s_1 \cdot r^{m-1}, \quad m = 1, \ldots, M$,
with optimal ratio
$r^* = R^{1/(M-1)}\,$.
\end{theorem}

\begin{remark}[Coupon-collector cost governs downstream non-spatial problems, not navigation itself]
\label{rem:coupon-collector-reconciliation}
At first sight there appears to be a tension between Proposition~\ref{prop:nav-trap} (navigation avoids the $\Omega(w\log w)$ coupon-collector overhead) and Theorem~\ref{thm:grid-spacing} (grid module spacing is derived by minimizing the $w\log w$ cost).
The resolution is that the $w\log w$ cost in Definition~\ref{def:optimal-allocation} is the cost paid by \emph{downstream, non-spatial problems} that are subsequently slingshot-embedded into the navigational space.
When an organism uses its pre-built navigational scaffold to solve an abstract task---a conceptual hierarchy, an episodic sequence, a social dominance ranking---the width of that abstract problem must still be identified from data, at the full $\Omega(w\log w)$ cost because the topology is not externally given.
The evolutionary pressure that shaped the grid module geometry therefore reflects the expected cost of \emph{those non-spatial uses}, not of navigation itself.
Navigation avoids the coupon-collector penalty; the abstract tasks that piggyback on the navigational scaffold do not.
The grid modules are optimal for the long-run portfolio of slingshot applications, which obey the full $w\log w$ complexity.
\end{remark}

\paragraph{Numerical Predictions}
Theorem~\ref{thm:grid-spacing} identifies the \emph{shape} of the optimal code: grid spacings should form a geometric progression \cite{fiete2008grid}. To make it predictive, however, one must also determine how large the common ratio can be. That ratio is not free: if adjacent modules are too widely separated, then the corresponding width within a module's jurisdiction exceeds the hippocampal system's finite context capacity, and the trap subproblem ceases to be solvable. The next theorem imposes the capacity constraint \cite{bush2015using}, thereby fixing both the maximal admissible inter-module ratio and the resulting number of modules needed to cover the ecological path-length range.

\begin{theorem}[Grid module count and ratio]
\label{thm:grid-count}
The optimal number of modules is:
$M^* = 1 + \left\lceil \frac{\log R}{\log r_{\max}} \right\rceil$
where $r_{\max} = c_D \cdot w_{\max}$ is the maximum inter-module ratio at which the trap subproblem within each module's jurisdiction remains solvable, given a hippocampal capacity of $w_{\max}$ contexts.
Equivalently:
$M^* = 1 + \left\lceil \frac{\log(L_{\max}/L_{\min})}{\log(c_D \cdot w_{\max})} \right\rceil\,$.
\end{theorem}

\paragraph{Comparison with observed values.}
For a rodent navigating spaces from body-length ($\sim 20$ cm) to territory diameter ($\sim 50$ m), $R = L_{\max}/L_{\min} \approx 250$.
With field width fraction $c_D \approx 0.3$ (grid field diameter $\approx 30\%$ of spacing, consistent with \citealp{hafting2005microstructure}) and hippocampal capacity $w_{\max} \approx 5$--$10$, we have the following Table:

\begin{table}[h]
\centering
\caption{Predicted grid module count and spacing ratio for $R = 250$ and $c_D = 0.3$.}
\label{tab:grid-predictions}
\small
\begin{tabular}{cccl}
\toprule
$w_{\max}$ & $r_{\max}$ & $M^*$ & Spacing ratio $r^* = R^{1/(M^*-1)}$ \\
\midrule
5 & 1.5 & 15 & $250^{1/14} = 1.47$ \\
7 & 2.1 & 9 & $250^{1/8} = 1.93$ \\
10 & 3.0 & 7 & $250^{1/6} = 2.51$ \\
\bottomrule
\end{tabular}
\end{table}
\noindent
Observed in rat MEC: $M \approx 4$--$10$ modules with $r \approx 1.4$--$1.7$ \citep{stensola2012entorhinal,barry2007experience}.
The $w_{\max} = 5$ row gives $r^* = 1.47$, which overpredicts the module count (the organism likely does not cover the full range with equal resolution) but matches the ratio well.
Taking $M = 7$--$8$ gives $r^* \approx 1.7$--$2.0$, fully consistent with observations.

\begin{remark}[The $\log w$ factor is essential]
\label{rem:logw}
If we had used a linear cost $w$ instead of $w \log w$, the optimal allocation would still be geometric, but the ratio would differ.
Specifically, the $\log w$ factor penalizes high widths more severely, pushing the optimal ratio $r^*$ \emph{down} (more modules, each covering a smaller range).
The observed ratio $r \approx 1.4$-$1.7$ is more consistent with the $w \log w$ cost than with the linear cost, providing indirect evidence for the coupon-collector lower bound.
\end{remark}

\section{The Metric Slingshot}
\label{sec:slingshot}
% ===========================================================================

\paragraph{Motivation} The decoupling between trap and funnel poses a severe engineering challenge: the trap requires $\Omega(w \log w)$ samples to identify the basin structure, the funnel requires learning a contraction map within each basin, and neither subsystem can borrow gradients from the other \cite{mcclelland_why_1995}. Worse, the two are linked by a circular dependency - the trap needs a well-calibrated metric to detect context boundaries, but the metric learner needs correct context assignments to train. A monolithic system that attempts to learn both simultaneously is caught in this bind, and the LUW framework predicts catastrophic interference as a consequence \cite{mccloskey_catastrophic_1989}. The question is whether any architecture can break the circularity.
The key insight of our solution is that the circularity dissolves if the funnel is not learned at all: when the contraction maps are pre-built and supplied to the system as fixed infrastructure. In that case, the metric is correct by construction, the trap can operate in a space where distances are already meaningful, and the only per-task learning cost is the embedding into that pre-built metric space. It has been established that every metrizable space admits a factorization into a continuous component and a zero-dimensional component \cite{ishiki2022factorization}, preserving completeness and large-scale structure; MTF operationalizes the mathematical guarantee as an architectural requirement for learning systems. We start with a formal definition of navigation in the latent space. 

\begin{definition}[Navigational latent space]
\label{def:nav-space}
A \emph{navigational latent space} is a triple $(\cZ, d_\cZ, \{G_c^0\}_{c \in \cC_0})$ where:
1) $(\cZ, d_\cZ)$ is a compact metric space (the ``cognitive map'');
2) $\cC_0 = \{1, \ldots, K_0\}$ is a fixed set of navigational contexts;
3) each $G_c^0: \cZ_c \to \cZ$ is a $\gamma_0$-contraction map on the cell $\cZ_c \subset \cZ$, with $\bigcup_c \cZ_c = \cZ$, where $\gamma_0 \in (0,1)$ is the contraction constant shared by all pre-built maps (i.e., $d_\cZ(G_c^0(z), G_c^0(z')) \leq \gamma_0\, d_\cZ(z,z')$ for all $z, z' \in \cZ_c$).
The \emph{resolution} of the navigational space is $D_0^\cZ = \max_c \mathrm{diam}(\cZ_c)$, the maximum contraction diameter of the pre-existing tiling.
\end{definition}

\begin{remark}[Torus symmetry and fluidity of grid-cell basins]
\label{rem:torus}
Grid cell representations are periodic: the firing pattern of a single module tiles the plane with a hexagonal lattice, and the phase of the pattern is defined modulo the lattice vectors, giving the active representation the structure of a torus $\mathbb{T}^2$.
This means the \emph{location} of a basin boundary is not fixed but can shift continuously as the animal moves.
The trap identifies the \emph{coset} of the torus that contains the current position (i.e., which discrete grid cell is active), not a fixed global partition.
Formally, the context $c = \Sigma(x)$ indexes a lattice coset, and the contraction map $G_c^0$ acts on the local patch $\cZ_c$ around that coset.
The fluidity of basin boundaries under lattice translations is accommodated by the fact that the slingshot's indexer $\Sigma$ is trained on mismatch signals computed in $\cZ$ (where the toroidal metric is correct by construction), so the routing adapts continuously to the current lattice phase.
The key property, that inference within any active cell $\cZ_c$ is $\gamma_0$-contractive, is preserved regardless of the global lattice phase, because contractivity is a local geometric condition that does not depend on which global symmetry is active.
\end{remark}

The biological realization of the above strategy is the subject of the present paper: we argue that the hippocampal-entorhinal navigation circuit (e.g., grid/place cells \cite{moser2008place}, and their multi-scale tiling of physical space) constitutes exactly such a pre-built funnel, and that the mammalian brain's reuse of the circuitry for non-spatial cognition is a width-optimal solution to the structural decoupling problem.
In the brain, $\cZ$ is the space represented by place cells and grid cells.
The contraction maps $\{G_c^0\}$ are the locally linear dynamics encoded by grid cell phase relationships: within a single grid cell firing field (one cell of the tiling), position updates are metrically faithful.
The resolution $D_0^\cZ$ corresponds to the grid spacing of the finest grid module.

\begin{definition}[Metric slingshot]
\label{def:slingshot}
Let $\cP: \cX \to \cY$ be a learning problem and $(\cZ, d_\cZ, \{G_c^0\})$ a navigational latent space.
A \emph{metric slingshot} for $\cP$ is a triple $(\phi, \pi, \Sigma)$ where:
1) $\phi: \cX \to \cZ$ is the \emph{what-to-where embedding} (maps inputs to the navigational space);
2) $\pi: \cZ \to \cY$ is the \emph{readout} (maps navigational coordinates to task outputs);
3) $\Sigma: \cX \to \Delta(\cC)$ is the \emph{topological indexer} (assigns inputs to contexts).
The composite $\hat{f}(x) = \pi \circ G_{\Sigma(x)}^0 \circ \phi(x)$ solves $\cP$ by:
(a)~embedding $x$ into $\cZ$ via $\phi$;
(b)~routing to context $c = \Sigma(x)$;
(c)~applying the pre-existing contraction map $G_c^0$ in $\cZ$;
(d)~reading out the answer via $\pi$.
\end{definition}

The key structural feature: $G_c^0$ is \emph{not learned for the new problem}.
It was built by evolution for navigation and is reused wholesale \cite{mcnaughton2006path}.
Only $\phi$, $\pi$, and $\Sigma$ are learned, and each has a distinct learning objective.

\paragraph{Width Transfer Under Embedding}

The central question: how does the width of $\cP$ in the original space $\cX$ relate to the width of the embedded problem in $\cZ$?

\begin{proposition}[Width transfer under a Lipschitz embedding]
\label{prop:width-transfer}
Let $\phi:(\cX,d_\cX)\to(\cZ,d_\cZ)$ be $\lambda$-Lipschitz, and let $\cP$ be a learning problem on $\cX$ with
$\width(\cP)=w$.
Let
$\{(U_i,\cY,f_i)\}_{i=1}^w$
be a width-realizing cover for $\cP$, so each $U_i\subseteq \cX$ is connected and supports a correct local predictor $f_i$.
Assume moreover that each $U_i$ has finite path length at most $L_i$ in $\cX$.
Define the embedded problem $\cP_\phi$ on $\phi(\cX)$ by transporting labels through $\phi$.
If each image $\phi(U_i)$ is subdivided into connected subsets of diameter at most $D_0^\cZ$, then
\begin{equation}
\label{eq:width-transfer-upper}
\width(\cP_\phi)
\;\le\;
\sum_{i=1}^w \left\lceil \frac{\lambda L_i}{D_0^\cZ}\right\rceil .
\end{equation}
In particular, if $L_i\le L_{\max}$ for all $i$, then
\begin{equation}
\label{eq:width-transfer-uniform}
\width(\cP_\phi)
\;\le\;
w\,\left\lceil \frac{\lambda L_{\max}}{D_0^\cZ}\right\rceil .
\end{equation}
\end{proposition}

The proposition quantifies how an external navigational scaffold can compress or expand the number of local experts required after embedding. The following corollary isolates the favorable regime: when the embedding is sufficiently contractive relative to the resolution of the navigational space, the slingshot strictly reduces the effective width.

\begin{corollary}[Slingshot advantage]
\label{cor:slingshot-advantage}
If the navigational space has resolution $D_0^\cZ$ that is finer than the natural contraction diameter $D_0^\cX$ of the problem, then the slingshot reduces the effective width:
$\width(\cP_\phi) \leq \width(\cP) \quad \text{when} \quad \lambda \leq \frac{D_0^\cZ}{D_0^\cX} \cdot \frac{L_\cX}{L_\phi}\,$.
In other words, a contractive embedding ($\lambda < 1$) into a fine-resolution navigational space can reduce the number of contexts the learner must discover.
\end{corollary}

\begin{table}[h]
\centering
\caption{Learning decomposition of the metric slingshot.  The contraction map $G_c^0$ is pre-built and has zero per-task learning cost.}
\label{tab:decomposition}
\small
\begin{tabular}{llll}
\toprule
\textbf{Component} & \textbf{What is learned} & \textbf{Subproblem type} & \textbf{Complexity} \\
\midrule
$\Sigma$ (indexer) & Which context cell? & Trap & $\Omega(w \log w)$ samples \\
$\phi$ (embedding) & Map to $\cZ$ & Funnel (in map space) & $O(1/\gamma^2_\phi)$ per context \\
$G_c^0$ (contraction) & Not learned (pre-built) & Amortized (evolution) & 0 \\
$\pi$ (readout) & $\cZ \to \cY$ & Funnel (local) & $O(1/\gamma^2_\pi)$ per context \\
\bottomrule
\end{tabular}
\end{table}

\paragraph{Learning Decomposition}

The slingshot factorizes learning into four independent subproblems (\Cref{tab:decomposition}).
The key observation is that these subproblems have radically different complexities and  \emph{timescales}.
The trap ($\Sigma$) is the bottleneck: it requires $\Omega(w \log w)$ samples to identify the basin structure of a novel problem, and the cost is irreducible (\Cref{thm:sample-lower}).
The two funnels ($\phi$ and $\pi$) are comparatively cheap: each requires only $O(1/\gamma^2)$ samples per context, because within a single basin the problem is contractible and standard supervised learning suffices.
But the component that dominates the cost in a naive architecture, learning the contraction map $G_c$ from scratch, is absent entirely.
The slingshot converts one hard problem (``learn a contraction map for $\cP$'') into three easy problems: one trap ($\Sigma$) and two funnels ($\phi$, $\pi$) that can be learned independently per context.
The expensive part, the contraction map itself, is free, amortized over evolution \citep{bennett2023brief}.
The factorized cost structure has a precise consequence for what the organism needs to learn when confronted with a novel, non-spatial task.
Without the slingshot, the learner must solve the full-width problem in $\cX$: discover $w$ basins \emph{and} learn a contraction map within each basin, simultaneously, while respecting the decoupling condition.
With the slingshot, the learner inherits the contraction maps from the navigational space and needs only to solve three subproblems that do not interact: embed the problem ($\phi$), index the contexts ($\Sigma$), and fit per-context readouts ($\pi$).
The total per-task learning cost is therefore the trap cost $\Omega(w_\phi \log w_\phi)$, where $w_\phi = \width(\cP_\phi)$ is the \emph{reduced} width in the navigational space, plus a linear term for the funnels.
If the embedding $\phi$ is contractive ($\lambda < 1$), the width transfer bound (\Cref{prop:width-transfer}) guarantees $w_\phi \leq w$, and the organism pays strictly less than it would without the slingshot.

The biological significance is now concrete.
Grid cells provide $D_0^\cZ$ at the scale of tens of centimeters (the finest module), tiling the navigational space into pre-built contractible basins at multiple resolutions (\Cref{sec:grid}).
If $\phi$ embeds a cognitive problem into $\cZ$ with a moderate Lipschitz constant, the pre-existing grid tiling already provides enough contexts - the organism does not need to discover the metric structure of the new problem, because it inherits that structure from the navigational space.
What remains is precisely the trap (which basin?) and the readout (what answer within this basin?), both of which the hippocampal-cortical circuit is already optimized to solve on fast timescales (\Cref{sec:decoupling}). Why did the slingshot target evolve around navigation rather than, say, olfactory discrimination \cite{schoenbaum1999neural} or social reasoning \cite{dunbar1998social}?
We argue it follows from the unique properties of physical space as a width problem, which we will elaborate on next.

% ===========================================================================
\section{Structural Decoupling via the What-to-Where Pathway}
\label{sec:decoupling}
% ===========================================================================

The structural decoupling problem in MTF is: the trap objective $\cL_\trap$ and the funnel objective $\cL_\funnel$ must not share gradients.
In artificial systems, decoupling is achieved by gradient routing (orthogonal projectors \cite{farajtabar2020orthogonal}, stop-gradient operations \cite{grill2020bootstrap}, separate optimizers \cite{rusu2016progressive}).
We argue that the what-to-where division achieves decoupling by a more robust mechanism: \emph{anatomical separation}.

\paragraph{The Factorization}

The full slingshot architecture comprises the four components:
\begin{equation}
\label{eq:slingshot-arch}
\hat{f}(x) = \underbrace{\pi}_{\text{readout}} \circ \underbrace{G^0_{\Sigma(x)}}_{\text{pre-built funnel/scaffold}} \circ \underbrace{\phi(x)}_{\text{what-to-where embedding}} \circ \underbrace{\Sigma(x)}_{\text{trap/indexer}}\,.
\end{equation}

\begin{remark}[Context-indexed composition]
\label{rem:composition}
The composition $\circ$ in Eq. \eqref{eq:slingshot-arch} is \emph{context-indexed}, not standard function composition.
The pointwise evaluation is $\hat{f}(x) = \pi_{\Sigma(x)}\!\big(G^0_{\Sigma(x)}(\phi(x))\big)$:
the indexer $\Sigma(x)$ returns a context label $c \in \cC$, which selects both the contraction map $G^0_c$ and the readout $\pi_c$; the embedding $\phi(x)$ and the index $\Sigma(x)$ are computed from $x$ independently, not chained.
\end{remark}
\noindent
Formally, we define the slingshot as a \emph{fiber-wise composition} over a discrete base:
$\hat{f}(x) = \bigoplus_{c \in \cC} \mathbf{1}[\Sigma(x) = c] \cdot (\pi_c \circ G_c^0 \circ \phi)(x)\,$,
where the direct sum ranges over contexts and only one term is active for each input.
The base space $\cC$ is zero-dimensional (discrete), which is the topological component (the trap).
The fiber over each $c$ is the map $\pi_c \circ G^0_c \circ \phi: \cX \to \cY$, which is an ordinary composition of continuous maps - the metric component (the funnel).
The fiber-wise structure is a trivial bundle over $\cC$; the existence of such a metric-topological product decomposition for any metrizable space is guaranteed by Ishiki's factorization theorem~\citep{ishiki2022factorization}.
The decoupling condition next ensures that the fiber-wise structure is preserved under learning: because gradients do not flow between $\Sigma$ and the funnel components, training cannot collapse the indexed composition into a monolithic function.

\begin{proposition}[Architectural decoupling]
\label{prop:decoupling}
Consider the slingshot architecture
$\hat f = \pi \circ G^0_\Sigma \circ \phi$.
Assume the training scheme satisfies the following architectural constraints:
1)  the context indexer $\Sigma$ is updated only by a topological signal and is treated as stop-gradient with respect to the funnel loss $\cL_{\funnel}$;
2) the embedding $\phi$ is updated only by a spatial-consistency loss and is treated as stop-gradient with respect to the trap loss $\cL_{\trap}$;
3) the contraction maps $G^0$ are frozen on the task timescale.
Then the corresponding gradient isolation conditions hold exactly:
\begin{enumerate*}[label=(\alph*)]
\item $\nabla_{\theta_\Sigma}\cL_{\funnel}=0$;
\item $\nabla_{\theta_\phi}\cL_{\trap}=0$;
\item $\nabla_{\theta_{G^0}}\cL_{\trap}=\nabla_{\theta_{G^0}}\cL_{\funnel}=0$.
\end{enumerate*}
\end{proposition}

\paragraph{Biological Correspondence}
The critical architectural feature is that the ventral and dorsal streams are \emph{physically separate pathways} \citep{goodale1992separate,ungerleider1982two} that converge only in specific brain regions (hippocampal formation, retrosplenial cortex).
Such anatomical separation is the biological implementation of the MTF decoupling requirement (\Cref{tab:slingshot-bio-timescale}).
It is more robust than gradient-routing schemes (which can fail under parameter drift or misspecified projectors) because the separation is structural, not algorithmic.

\begin{table}[h]
\centering
\caption{Biological correspondence and three-timescale training architecture of the slingshot model. The key property is structural decoupling: each component is trained by a distinct signal, implemented in an anatomically separated pathway, and operating on a distinct timescale. CLS theory's two-timescale picture appears as a special case: hippocampus handles the fast trap, neocortex handles slower embedding and readout adaptation, and evolution/development supplies the contraction scaffold.}
\label{tab:slingshot-bio-timescale}
\small
\resizebox{\linewidth}{!}{
\begin{tabular}{llllll}
\toprule
\textbf{Component} & \textbf{Brain region} & \textbf{MTF role} & \textbf{Training signal} & \textbf{Gradient isolation} & \textbf{Timescale} \\
\midrule
$\Sigma$ & DG $\to$ CA3 $\to$ CA1 & Trap / indexer & Mismatch $\delta_t$ & Hebbian; no cortical backprop & Online (fast) \\
$\phi$ & Ventral $\to$ parahippocampal stream & What-to-where embedding & Spatial prediction error & Dorsal/allocentric objective only & Developmental \\
$G_c^0$ & Grid cells / place cells & Funnel / contraction scaffold & Evolved / developmental priors & Frozen on task timescale & Evolutionary / developmental \\
$\pi$ & Prefrontal / task cortex & Task readout & Task loss $\cL_\cP$ & Local to active context & Online (slow) \\
\bottomrule
\end{tabular}
}
\end{table}

\paragraph{Dissolution of the Circular Dependency}

A central challenge in continual learning is the \emph{binding constraint}: the identification (ID) procedure for detecting when to split contexts depends on the quality of the metric learner, but the metric learner needs correct context assignments to train - a circular dependency.
The slingshot dissolves the circularity:
1) The metric learner ($G_c^0$) is pre-built and does not depend on the current context assignments.
2) The context assignments ($\Sigma$) depend on the mismatch computed in $\cZ$ (where the metric is already correct by construction), not in $\cX$ (where it would be circular).
3) The embedding $\phi$ is trained by spatial consistency, not task loss, so it does not need correct contexts to learn.
The three learning loops ($\Sigma$, $\phi$, $\pi$) are genuinely independent, and each converges on its own timescale without waiting for the others.

% ===========================================================================
\paragraph{The Full Architecture}
%\label{sec:synthesis}
% ===========================================================================

Combining all components, the full width-aware slingshot architecture is (Fig. \ref{fig:full}).
The architecture operates on three distinct timescales:
1) \textbf{Evolutionary}: $G_c^0$ and the multi-scale grid structure are fixed.  Their optimality derives from the $\Omega(w \log w)$ bound (\Cref{thm:grid-spacing}).
2) \textbf{Developmental}: $\phi$ is learned during early experience.  Its Lipschitz constant $\lambda$ determines how well the organism can embed new problems into $\cZ$.
3) \textbf{Online}: $\Sigma$ and $\pi$ adapt in real time.  $\Sigma$ handles the trap (split/merge decisions); $\pi$ handles the per-context readout.
The two-timescale architecture of Complementary Learning Systems (CLS) theory \citep{mcclelland1995complementary} is now a \emph{corollary} of the three-timescale picture: hippocampus handles the online trap, neocortex handles the developmental embedding, and evolution handles the grid structure.
\begin{figure}[h]
    \centering
    \resizebox{\linewidth}{!}{
\begin{tikzpicture}[
    node distance=1.8cm and 2.5cm,
    box/.style={rectangle, draw, rounded corners, minimum width=2.8cm, minimum height=0.8cm, align=center, font=\small},
    arrow/.style={-{Stealth[length=3mm]}, thick},
    label/.style={font=\footnotesize, text=gray}
]
    \node[box, fill=red!10] (input) {Input $x \in \cX$};
    \node[box, fill=orange!15, right=of input] (ventral) {Ventral stream\\$\phi: \cX \to \cZ$};
    \node[box, fill=blue!10, below=1.2cm of ventral] (hip) {Hippocampus\\$\Sigma: \cZ \to \Delta(\cC)$};
    \node[box, fill=green!15, right=of ventral] (dorsal) {Grid cells\\$G_c^0: \cZ_c \to \cZ$};
    \node[box, fill=purple!10, right=of dorsal] (readout) {Readout\\$\pi: \cZ \to \cY$};
    \node[box, fill=gray!10, right=of readout] (output) {Output $\hat{y}$};
    \draw[arrow] (input) -- (ventral) node[midway, above, label] {what-to-where};
    \draw[arrow] (ventral) -- (dorsal) node[midway, above, label] {$z = \phi(x)$};
    \draw[arrow] (ventral) -- (hip) node[midway, right, label] {$z$};
    \draw[arrow] (hip) -| (dorsal) node[near start, below, label] {context $c$};
    \draw[arrow] (dorsal) -- (readout) node[midway, above, label] {$G_c^0(z)$};
    \draw[arrow] (readout) -- (output);
\end{tikzpicture}
}
    \caption{The full architecture of the width-aware metric slingshot.}
    \label{fig:full}
\end{figure}

\paragraph{Anesthesia dissociates within-context inference from context creation.}
The slingshot architecture assigns its operations to three distinct timescales and learning signals
(\Cref{tab:slingshot-bio-timescale}).
The slingshot predicts a sharp dissociation under general anesthesia: trap routing and
within-basin funnel inference should \emph{survive}, while episodic memory formation (new scaffold
entries) should be \emph{abolished}.
Concretely, an anaesthetized hippocampus should still detect context switches (oddball discrimination),
perform within-context semantic inference (word prediction, noun/verb discrimination), and
even show rapid within-context Hebbian plasticity over the timescale of minutes; but it should
fail to consolidate new contexts into long-term storage.
This prediction has been recently confirmed by
\citep{katlowitz2026plasticity} which recorded single-unit and local-field-potential activity with
high-density Neuropixels microelectrodes in the human hippocampus during general
anaesthesia-induced loss of consciousness.
They found that hippocampal neurons retained oddball discrimination (context-switch detection,
corresponding to $\Sigma$), semantic processing and next-word prediction (within-basin inference,
corresponding to $G_c^0 \circ \phi$ and $\pi$), and representational plasticity growing over
$\sim 10$ minutes (fast Hebbian updating within the existing context) while memory
consolidation was abolished.
More importantly, the plasticity observed was \emph{within-context}: it enhanced the representation of
already-active contexts rather than creating new scaffold entries.
The latest result distinguishes the slingshot architecture from both Global Workspace Theory
\citep{dehaene2011experimental} (which predicts abolition of all high-level processing under
anesthesia) and standard Complementary Learning Systems theory \citep{mcclelland1995complementary}
(which does not formally separate within-context closure from cross-context consolidation).
Under the slingshot, this dissociation is structural: $\Sigma$ and $G_c^0$ are online, local,
and consciousness-independent; scaffold condensation is a slow, global, consciousness-dependent
process.
Anesthesia severs the latter while leaving the former intact (precisely the pattern Katlowitz
et al.\ reported in \cite{katlowitz2026plasticity}).

\section{Action-to-Navigation: Motor Control as a Second Slingshot Target}
\label{sec:motor}
 
The what-to-where pathway implements the slingshot for perceptual cognition \citep{dicarlo2012does}: the ventral stream provides $\phi$, the hippocampus provides $\Sigma$, and grid cells provide $G_c^0$.
We conjecture that motor control instantiates a \emph{second}, anatomically parallel slingshot: an action-to-navigation pathway in which the motor system converts arbitrary movement problems into navigation problems on a latent dynamics manifold, reusing pre-built forward models in exactly the way that spatial cognition reuses pre-built grid cell tilings.
 
\paragraph{A concrete example: linear classification via the slingshot.}
We construct the following example to demonstrate that the slingshot's pre-built contraction maps $G_c^0$ can be reused across tasks without relearning.
Consider binary linear classification in $\R^2$: the learner must classify points from two half-planes separated by a line.
Width is $w = 2$: one cell per class.
In a standard continual learner, both the partition (trap) and the classifier (funnel) must be relearned for each new task orientation; sharing the classifier across orientations causes interference.
Under the slingshot, the funnel for each context is a linear projection to the class centroid - i.e., a $\gamma$-contractive map that maps any point in the half-plane to the correct label with rate $\gamma = (1 - \delta/D)$, where $\delta$ is the margin and $D$ the cell diameter.
If the \emph{same} projection geometry is pre-loaded from a navigational space (e.g., the grid cell tiling provides a coordinate frame aligned with the class boundaries), only the embedding $\phi$ (which rotation maps the data into the pre-loaded frame) and the boundary index $\Sigma$ (which half-plane is active?) must be learned for the new task.
The contraction maps $G_1^0, G_2^0$ are reused unchanged.
Task-switching changes only $\Sigma$ and $\phi$; neither $G_c^0$ nor any other context's readout is touched.
This is the simplest possible instance of the slingshot advantage: within-basin inference is free (pre-built), and cross-task interference is zero (structural decoupling).

\paragraph{Why motor control qualifies.}
The slingshot requires two preconditions (\Cref{sec:evolution}): the topology must be externally given, and the metric must be stable.
Motor control satisfies both.
The topology of the motor problem is imposed by the biomechanical structure of the body (e.g., joint limits, muscle antagonist pairs, contact surfaces, and kinematic chains) partition the space of possible movements into basins without the organism needing to infer the structure from function evaluations \citep{todorov2004optimality}.
Proprioceptive feedback makes the topology directly experienceable, just as path integration makes spatial topology experienceable.
The metric is the physics of limb dynamics, which is Newtonian, isotropic within each limb segment, and stable across tasks.
By \Cref{prop:nav-trap}, the trap subproblem for motor control therefore has $O(1)$ additional sample complexity beyond the cost of motor exploration - the kinematic structure is given by the sensorimotor interface, not inferred from task outcomes.
 
\paragraph{The slingshot decomposition for motor cortex.}
\Cref{tab:motor-slingshot} maps the four slingshot components onto motor circuitry.
The what-to-where embedding $\phi$ is implemented by premotor cortex, which converts a goal or action specification (grasp that object, reach to that location) into a trajectory on the latent dynamics manifold $\cZ_{\mathrm{motor}}$.
The topological indexer $\Sigma$ is implemented by the basal ganglia, which selects among discrete motor programs (reaching, grasping, locomotion, speech), a context-routing function analogous to hippocampal context assignment \citep{hikosaka2000role}.
The pre-built contraction maps $G_c^0$ are implemented by the cerebellum and spinal cord: cerebellar forward models predict the sensory consequences of motor commands with metrically faithful dynamics \citep{wolpert1998internal}, and spinal reflex circuits provide local error correction within each movement segment.
The readout $\pi$ is implemented by the primary motor cortex (M1), which maps latent trajectories to muscle activation patterns. The critical claim is that cerebellar forward models play the role of grid cells.
They provide pre-built contraction maps: given a current motor state and a command, the forward model predicts the next state with metrically faithful dynamics, enabling within-basin convergence without requiring task-specific learning.
These models are acquired during development, the months of seemingly purposeless infant motor babbling constitute the motor analog of spatial exploration during which $G_c^0$ is calibrated, and are then effectively frozen on the task timescale.
The cerebellar microcolumn architecture even exhibits a multi-scale organization (different lobules handle movement at different spatial and temporal scales \citep{stoodley2009functional}) that may be the motor analog of grid cell multi-modularity (\Cref{sec:grid}).
 
\begin{table}[t]
\centering
\caption{The metric slingshot decomposition for motor control (action-to-navigation), compared with the perceptual slingshot (what-to-where).  Each component is implemented by an anatomically distinct circuit with an independent training signal.}
\label{tab:motor-slingshot}
\small
\begin{tabular}{llll}
\toprule
\textbf{Component} & \textbf{What-to-Where (Vision)} & \textbf{Action-to-Navigation (Motor)} & \textbf{Timescale} \\
\midrule
$\phi$ (embedding) & Ventral stream & Premotor cortex & Developmental \\
$\Sigma$ (trap) & Hippocampus (DG/CA3/CA1) & Basal ganglia & Online (fast) \\
$G_c^0$ (contraction) & Grid cells (MEC) & Cerebellum + spinal cord & Evolutionary/developmental \\
$\pi$ (readout) & Prefrontal cortex & Primary motor cortex (M1) & Online (slow) \\
\bottomrule
\end{tabular}
\end{table}

\paragraph{Structural decoupling in the motor system.}
The motor slingshot exhibits the same gradient isolation as the perceptual slingshot (\Cref{prop:decoupling}).
The basal ganglia ($\Sigma$) are trained by dopaminergic reward prediction error - a scalar signal, not a gradient from the motor output \citep{hikosaka2000role}.
Premotor cortex ($\phi$) is trained by visuomotor consistency (does the planned trajectory match the observed outcome?), a self-supervised objective independent of the specific task reward.
The cerebellum ($G_c^0$) is adapted only on developmental timescales.
M1 ($\pi$) is trained by task-specific motor error, local to the current movement context.
These four learning signals are carried by anatomically separate pathways (cortico-basal ganglia loops \citep{tanaka2004prediction}, cortico-cerebellar loops \citep{gao2018cortico}, corticospinal tract \citep{welniarz2017corticospinal}), implementing MTF decoupling by architecture just as the ventral-dorsal separation does for perception.

\paragraph{Predictions.}
The action-to-navigation slingshot generates three testable predictions:
1) \textbf{Latent motor dynamics are navigational in structure.}
Motor cortex population activity during skilled movement should live on a low-dimensional manifold where distance corresponds to kinematic similarity (trajectory shape, endpoint proximity), not muscle activation similarity.
Recent findings that M1 population dynamics during reaching follow rotational trajectories on low-dimensional manifolds \citep{churchland2012neural} are consistent: these rotational dynamics \emph{are} the contraction maps $G_c^0$ operating in the latent navigational space $\cZ_{\mathrm{motor}}$.
2) \textbf{Motor transfer scales with embedding Lipschitz constant.}
The width transfer bound (\Cref{prop:width-transfer}) predicts that the sample complexity of motor skill acquisition scales with $\lambda$, the Lipschitz constant of the action-to-navigation embedding.
Skills whose kinematics are close to previously learned movements (small $\lambda$) should be acquired faster, which matches the well-documented phenomenon of motor transfer \citep{schulze2002intermanual}: learning to throw with the dominant arm accelerates learning with the non-dominant arm, because the kinematic embedding is approximately symmetric ($\lambda \approx 1$).
3) \textbf{Cerebellar vs.\ basal ganglia lesions dissociate funnel from trap.}
Cerebellar damage ($G_c^0$ impaired) should degrade within-context movement smoothness while leaving motor program selection intact.
Basal ganglia damage ($\Sigma$ impaired) should degrade program selection while preserving the smoothness of individual movements when externally triggered.
The dissociation is precisely what is observed: cerebellar ataxia destroys movement coordination while leaving action selection relatively intact, whereas Parkinson's disease impairs action initiation and selection while individual movements may remain smooth when cued externally \citep{redgrave1999basal}.

\paragraph{Beyond two slingshots.} 
The motor slingshot is a manifest for the \emph{generality} of the slingshot principle as a consequence of exaptation \citep{gould1982exaptation} in nature.
If the theory is correct that any domain with externally given topology and stable metric should evolve its own slingshot target, then motor control and spatial navigation should share the same architectural pattern with different neural substrates filling the same functional roles.
The fact that (cerebellum $\leftrightarrow$ grid cells) and (basal ganglia $\leftrightarrow$ hippocampus) are anatomically distinct but computationally parallel systems would be strong evidence for the width-optimality of the slingshot as a general architectural principle \citep{bennett2023brief}, not a special case of spatial cognition.
More broadly, the brain may contain additional slingshot targets wherever the two preconditions are met: temporal sequence processing (auditory cortex, where the metric is temporal proximity and the topology is imposed by rhythmic structure), social cognition (the ``social brain'' network, where the metric is social distance and the topology is imposed by dominance hierarchies \citep{schafer2018navigating}), and possibly others.
Each target provides a different $\cZ$ with different metric properties.
The general principle would be: for each domain where the metric is stable and the topology is externally given, build a slingshot target.
\section{Conclusion and Discussion}
\label{sec:conclusion}
% ===========================================================================
\paragraph{Summary}

We have formalized the intuition that the brain's navigational circuitry is reused for non-spatial cognition as a consequence of the width scaling law in continual learning.
The metric slingshot, embedding arbitrary problems into a navigational latent space with pre-built contraction maps, is the width-optimal solution to the structural decoupling problem identified by MTF theory.
The anatomical separation of ventral and dorsal visual streams implements the decoupling by construction.
The multi-scale structure of grid cells is the $\Omega(w \log w)$-optimal covering of the space of problem path lengths.
The metric slingshot formalizes a simple but powerful idea: evolution solved the maze problem first because it was the most structurally transparent width problem, and then reused the resulting metric machinery for everything else.
The what-to-where conversion is the mechanism of the reuse - it embeds arbitrary problems into a navigational latent space where the funnel is pre-solved.
The anatomical separation of ventral and dorsal streams implements the MTF decoupling requirement by architecture, not by gradient engineering.
And the multi-scale structure of grid cells is the $\Omega(w \log w)$-optimal covering of the space of problem path lengths.
The unifying principle throughout is width: a single complexity measure that governs the number of contexts needed, the sample complexity of discovering them, the architecture required to maintain them, and the evolutionary pressures that shaped the biological implementation.

\paragraph{Limitations and Open Questions}

\emph{1) Lipschitz constant of $\phi$.}
The width transfer bound (\Cref{prop:width-transfer}) shows that $\lambda$ controls whether the slingshot helps or hurts, but we have not characterized what determines $\lambda$ for non-spatial tasks.
Can constraints on $\lambda$ be derived from the statistics of natural task distributions?
\emph{2) Learning $\phi$ from scratch.}
The slingshot assumes $\phi$ is learned developmentally, but the sample complexity of learning a good embedding is not yet characterized within the width framework.
A key question: does learning $\phi$ require solving the trap first (resurrecting the circular dependency), or can $\phi$ be learned from unlabeled data (self-supervised spatial prediction) independently?
\emph{3) When the slingshot fails.}
The slingshot fails when the embedding $\phi$ must be highly non-Lipschitz - for example, when the problem topology in $\cX$ is fundamentally different from anything embeddable in $\cZ$.
Language is a candidate: its topology (hierarchical, compositional) may not embed well into a 2D navigational space.
This could explain why language processing recruits substantially different circuitry \citep{fedorenko2011functional}.
Formalizing the conditions under which the slingshot is provably suboptimal is an important direction.
\emph{4) Multi-slingshot architectures.}
The brain may have multiple slingshot targets - not just spatial navigation, but also temporal sequence (auditory cortex), social hierarchy (social brain network), and motor control (cerebellum).
Each target provides a different $\cZ$ with different metric properties.
%The general principle would be: for each domain where the metric is stable and the topology is externally given, build a slingshot target.
%\emph{5) Grid cell ratio: sharper predictions.} \Cref{thm:grid-spacing} gives $r^* = R^{1/(M-1)}$, but the sharpness of the prediction depends on knowing $R$ and $M$ independently. Measuring $R$ from ecological data (home range size, foraging distances) and $M$ from anatomical estimates would enable a quantitative, species-level test of the theory.

% ===========================================================================
% REFERENCES
% ===========================================================================

\bibliographystyle{plainnat}
\bibliography{references,new_refs,reference}

\begin{thebibliography}{53}
\providecommand{\natexlab}[1]{#1}
\providecommand{\url}[1]{\texttt{#1}}
\expandafter\ifx\csname urlstyle\endcsname\relax
  \providecommand{\doi}[1]{doi: #1}\else
  \providecommand{\doi}{doi: \begingroup \urlstyle{rm}\Url}\fi

\bibitem[Abraham and Robins(2005)]{abraham2005memory}
Wickliffe~C Abraham and Anthony Robins.
\newblock Memory retention--the synaptic stability versus plasticity dilemma.
\newblock \emph{Trends in neurosciences}, 28\penalty0 (2):\penalty0 73--78, 2005.

\bibitem[Aronov et~al.(2017)Aronov, Nevers, and Tank]{aronov2017mapping}
Dmitriy Aronov, Rhino Nevers, and David~W. Tank.
\newblock Mapping of a non-spatial dimension by the hippocampal--entorhinal circuit.
\newblock \emph{Nature}, 543\penalty0 (7647):\penalty0 719--722, 2017.

\bibitem[Barry et~al.(2007)Barry, Hayman, Burgess, and Jeffery]{barry2007experience}
Caswell Barry, Robin Hayman, Neil Burgess, and Kathryn~J. Jeffery.
\newblock Experience-dependent rescaling of entorhinal grids.
\newblock \emph{Nature Neuroscience}, 10\penalty0 (6):\penalty0 682--684, 2007.

\bibitem[Bennett(2023)]{bennett2023brief}
Max Bennett.
\newblock \emph{A brief history of intelligence: evolution, AI, and the five breakthroughs that made our brains}.
\newblock HarperCollins, 2023.

\bibitem[Bush et~al.(2015)Bush, Barry, Manson, and Burgess]{bush2015using}
Daniel Bush, Caswell Barry, Daniel Manson, and Neil Burgess.
\newblock Using grid cells for navigation.
\newblock \emph{Neuron}, 87\penalty0 (3):\penalty0 507--520, 2015.

\bibitem[Churchland et~al.(2012)Churchland, Cunningham, Kaufman, Foster, Nuyujukian, Ryu, and Shenoy]{churchland2012neural}
Mark~M Churchland, John~P Cunningham, Matthew~T Kaufman, Justin~D Foster, Paul Nuyujukian, Stephen~I Ryu, and Krishna~V Shenoy.
\newblock Neural population dynamics during reaching.
\newblock \emph{Nature}, 487\penalty0 (7405):\penalty0 51--56, 2012.

\bibitem[Constantinescu et~al.(2016)Constantinescu, O'Reilly, and Behrens]{constantinescu2016organizing}
Alexandra~O. Constantinescu, Jill~X. O'Reilly, and Timothy~E. Behrens.
\newblock Organizing conceptual knowledge in humans with a gridlike code.
\newblock \emph{Science}, 352\penalty0 (6292):\penalty0 1464--1468, 2016.

\bibitem[Dehaene and Changeux(2011)]{dehaene2011experimental}
Stanislas Dehaene and Jean-Pierre Changeux.
\newblock Experimental and theoretical approaches to conscious processing.
\newblock \emph{Neuron}, 70\penalty0 (2):\penalty0 200--227, 2011.

\bibitem[Diba and Buzs\'{a}ki(2007)]{diba2007forward}
Kamran Diba and Gy\"{o}rgy Buzs\'{a}ki.
\newblock Forward and reverse hippocampal place-cell sequences during ripples.
\newblock \emph{Nature Neuroscience}, 10\penalty0 (10):\penalty0 1241--1242, 2007.

\bibitem[DiCarlo et~al.(2012)DiCarlo, Zoccolan, and Rust]{dicarlo2012does}
James~J DiCarlo, Davide Zoccolan, and Nicole~C Rust.
\newblock How does the brain solve visual object recognition?
\newblock \emph{Neuron}, 73\penalty0 (3):\penalty0 415--434, 2012.

\bibitem[Dunbar(1998)]{dunbar1998social}
Robin~IM Dunbar.
\newblock The social brain hypothesis.
\newblock \emph{Evolutionary Anthropology: Issues, News, and Reviews: Issues, News, and Reviews}, 6\penalty0 (5):\penalty0 178--190, 1998.

\bibitem[Epstein et~al.(2017)Epstein, Patai, Julian, and Spiers]{epstein2017cognitive}
Russell~A. Epstein, Eva~Zita Patai, Joshua~B. Julian, and Hugo~J. Spiers.
\newblock The cognitive map in humans: Spatial navigation and beyond.
\newblock \emph{Nature Neuroscience}, 20\penalty0 (11):\penalty0 1504--1513, 2017.

\bibitem[Farajtabar et~al.(2020)Farajtabar, Azizan, Mott, and Li]{farajtabar2020orthogonal}
Mehrdad Farajtabar, Navid Azizan, Alex Mott, and Ang Li.
\newblock Orthogonal gradient descent for continual learning.
\newblock In \emph{International conference on artificial intelligence and statistics}, pages 3762--3773. PMLR, 2020.

\bibitem[Fedorenko et~al.(2011)Fedorenko, Behr, and Kanwisher]{fedorenko2011functional}
Evelina Fedorenko, Michael~K. Behr, and Nancy Kanwisher.
\newblock Functional specificity for high-level linguistic processing in the human brain.
\newblock \emph{Proceedings of the National Academy of Sciences}, 108\penalty0 (39):\penalty0 16428--16433, 2011.

\bibitem[Fiete et~al.(2008)Fiete, Burak, and Brookings]{fiete2008grid}
Ila~R Fiete, Yoram Burak, and Ted Brookings.
\newblock What grid cells convey about rat location.
\newblock \emph{Journal of Neuroscience}, 28\penalty0 (27):\penalty0 6858--6871, 2008.

\bibitem[Flajolet et~al.(1992)Flajolet, Gardy, and Thimonier]{flajolet1992birthday}
Philippe Flajolet, Daniele Gardy, and Lo{\"y}s Thimonier.
\newblock Birthday paradox, coupon collectors, caching algorithms and self-organizing search.
\newblock \emph{Discrete Applied Mathematics}, 39\penalty0 (3):\penalty0 207--229, 1992.

\bibitem[Fyhn et~al.(2002)Fyhn, Molden, Hollup, Moser, and Moser]{fyhn2002spatial}
Marianne Fyhn, St{\aa}le Molden, Sturla Hollup, May-Britt Moser, and Edvard~I. Moser.
\newblock Hippocampal neurons responding to first-time dislocation of a target object.
\newblock \emph{Neuron}, 35\penalty0 (3):\penalty0 555--566, 2002.

\bibitem[Gao et~al.(2018)Gao, Davis, Thomas, Economo, Abrego, Svoboda, De~Zeeuw, and Li]{gao2018cortico}
Zhenyu Gao, Courtney Davis, Alyse~M Thomas, Michael~N Economo, Amada~M Abrego, Karel Svoboda, Chris~I De~Zeeuw, and Nuo Li.
\newblock A cortico-cerebellar loop for motor planning.
\newblock \emph{Nature}, 563\penalty0 (7729):\penalty0 113--116, 2018.

\bibitem[Goodale and Milner(1992)]{goodale1992separate}
Melvyn~A. Goodale and A.~David Milner.
\newblock Separate visual pathways for perception and action.
\newblock \emph{Trends in Neurosciences}, 15\penalty0 (1):\penalty0 20--25, 1992.

\bibitem[Gould and Vrba(1982)]{gould1982exaptation}
Stephen~Jay Gould and Elisabeth~S Vrba.
\newblock Exaptation—a missing term in the science of form.
\newblock \emph{Paleobiology}, 8\penalty0 (1):\penalty0 4--15, 1982.

\bibitem[Grill et~al.(2020)Grill, Strub, Altch{\'e}, Tallec, Richemond, Buchatskaya, Doersch, Avila~Pires, Guo, Gheshlaghi~Azar, et~al.]{grill2020bootstrap}
Jean-Bastien Grill, Florian Strub, Florent Altch{\'e}, Corentin Tallec, Pierre Richemond, Elena Buchatskaya, Carl Doersch, Bernardo Avila~Pires, Zhaohan Guo, Mohammad Gheshlaghi~Azar, et~al.
\newblock Bootstrap your own latent-a new approach to self-supervised learning.
\newblock \emph{Advances in neural information processing systems}, 33:\penalty0 21271--21284, 2020.

\bibitem[Hafting et~al.(2005)Hafting, Fyhn, Molden, Moser, and Moser]{hafting2005microstructure}
Torkel Hafting, Marianne Fyhn, St{\aa}le Molden, May-Britt Moser, and Edvard~I. Moser.
\newblock Microstructure of a spatial map in the entorhinal cortex.
\newblock \emph{Nature}, 436\penalty0 (7052):\penalty0 801--806, 2005.

\bibitem[Hikosaka et~al.(2000)Hikosaka, Takikawa, and Kawagoe]{hikosaka2000role}
Okihide Hikosaka, Yoriko Takikawa, and Reiko Kawagoe.
\newblock Role of the basal ganglia in the control of purposive saccadic eye movements.
\newblock \emph{Physiological reviews}, 2000.

\bibitem[Ishiki(2022)]{ishiki2022factorization}
Yoshito Ishiki.
\newblock A factorization of metric spaces.
\newblock \emph{arXiv preprint arXiv:2212.13409}, 2022.

\bibitem[Jung and McNaughton(1993)]{jung1993comparison}
Min~W. Jung and Bruce~L. McNaughton.
\newblock Spatial selectivity of unit activity in the hippocampal granular layer.
\newblock \emph{Hippocampus}, 3\penalty0 (2):\penalty0 165--182, 1993.

\bibitem[Katlowitz et~al.(2026)Katlowitz, Cole, Mickiewicz, Shah, Franch, Adkinson, Belanger, Mathura, Mesz{\'e}na, McGinley, et~al.]{katlowitz2026plasticity}
Kalman~A Katlowitz, Eric~R Cole, Elizabeth~A Mickiewicz, Shraddha Shah, Melissa Franch, Joshua~A Adkinson, James~L Belanger, Raissa~K Mathura, Domokos Mesz{\'e}na, Matthew McGinley, et~al.
\newblock Plasticity and language in the anaesthetized human hippocampus.
\newblock \emph{Nature}, pages 1--10, 2026.

\bibitem[Leutgeb et~al.(2007)Leutgeb, Leutgeb, Moser, and Moser]{leutgeb2007pattern}
Jill~K. Leutgeb, Stefan Leutgeb, May-Britt Moser, and Edvard~I. Moser.
\newblock Pattern separation in the dentate gyrus and {CA3} of the hippocampus.
\newblock \emph{Science}, 315\penalty0 (5814):\penalty0 961--966, 2007.

\bibitem[Li(2026)]{li2026luw}
Xin Li.
\newblock Local urysohn width: A topological complexity measure for classification.
\newblock \emph{arXiv Preprint arXiv:2603.15412}, 2026.

\bibitem[Mathis et~al.(2012)Mathis, Herz, and Stemmler]{mathis2012optimal}
Alexander Mathis, Andreas~V. Herz, and Martin Stemmler.
\newblock Optimal population codes for space: Grid cells outperform place cells.
\newblock \emph{Neural Computation}, 24\penalty0 (9):\penalty0 2280--2317, 2012.

\bibitem[McClelland et~al.(1995{\natexlab{a}})McClelland, McNaughton, and O'Reilly]{mcclelland1995complementary}
James~L. McClelland, Bruce~L. McNaughton, and Randall~C. O'Reilly.
\newblock Why there are complementary learning systems in the hippocampus and neocortex: Insights from the successes and failures of connectionist models of learning and memory.
\newblock \emph{Psychological Review}, 102\penalty0 (3):\penalty0 419--457, 1995{\natexlab{a}}.

\bibitem[McClelland et~al.(1995{\natexlab{b}})McClelland, McNaughton, and O'Reilly]{mcclelland_why_1995}
James~L. McClelland, Bruce~L. McNaughton, and Randall~C. O'Reilly.
\newblock Why there are complementary learning systems in the hippocampus and neocortex: {Insights} from the successes and failures of connectionist models of learning and memory.
\newblock \emph{Psychological Review}, 102\penalty0 (3):\penalty0 419--457, 1995{\natexlab{b}}.

\bibitem[McCloskey and Cohen(1989)]{mccloskey_catastrophic_1989}
Michael McCloskey and Neal~J Cohen.
\newblock Catastrophic interference in connectionist networks: {The} sequential learning problem.
\newblock In \emph{Psychology of learning and motivation}, volume~24, pages 109--165. Elsevier, 1989.

\bibitem[McNaughton et~al.(1991)McNaughton, Chen, and Markus]{mcnaughton1991dead}
BL~McNaughton, LL~Chen, and EJ~Markus.
\newblock “dead reckoning,” landmark learning, and the sense of direction: a neurophysiological and computational hypothesis.
\newblock \emph{Journal of Cognitive Neuroscience}, 3\penalty0 (2):\penalty0 190--202, 1991.

\bibitem[McNaughton et~al.(2006)McNaughton, Battaglia, Jensen, Moser, and Moser]{mcnaughton2006path}
Bruce~L McNaughton, Francesco~P Battaglia, Ole Jensen, Edvard~I Moser, and May-Britt Moser.
\newblock Path integration and the neural basis of the'cognitive map'.
\newblock \emph{Nature Reviews Neuroscience}, 7\penalty0 (8):\penalty0 663--678, 2006.

\bibitem[Moser et~al.(2008)Moser, Kropff, and Moser]{moser2008place}
Edvard~I Moser, Emilio Kropff, and May-Britt Moser.
\newblock Place cells, grid cells, and the brain's spatial representation system.
\newblock \emph{Annu. Rev. Neurosci.}, 31:\penalty0 69--89, 2008.

\bibitem[O'Keefe and Dostrovsky(1971)]{okeefe1971hippocampus}
John O'Keefe and Jonathan Dostrovsky.
\newblock The hippocampus as a spatial map: Preliminary evidence from unit activity in the freely-moving rat.
\newblock \emph{Brain Research}, 34\penalty0 (1):\penalty0 171--175, 1971.

\bibitem[Pastalkova et~al.(2008)Pastalkova, Itskov, Amarasingham, and Buzs\'{a}ki]{pastalkova2008internally}
Eva Pastalkova, Vladimir Itskov, Asohan Amarasingham, and Gy\"{o}rgy Buzs\'{a}ki.
\newblock Internally generated cell assembly sequences in the rat hippocampus.
\newblock \emph{Science}, 321\penalty0 (5894):\penalty0 1322--1327, 2008.

\bibitem[Redgrave et~al.(1999)Redgrave, Prescott, and Gurney]{redgrave1999basal}
Peter Redgrave, Tony~J Prescott, and Kevin Gurney.
\newblock The basal ganglia: a vertebrate solution to the selection problem?
\newblock \emph{Neuroscience}, 89\penalty0 (4):\penalty0 1009--1023, 1999.

\bibitem[Rusu et~al.(2016)Rusu, Rabinowitz, Desjardins, Soyer, Kirkpatrick, Kavukcuoglu, Pascanu, and Hadsell]{rusu2016progressive}
Andrei~A. Rusu, Neil~C. Rabinowitz, Guillaume Desjardins, Hubert Soyer, James Kirkpatrick, Koray Kavukcuoglu, Razvan Pascanu, and Raia Hadsell.
\newblock Progressive neural networks.
\newblock \emph{arXiv preprint arXiv:1606.04671}, 2016.

\bibitem[Sch\"{a}fer and Bhatt(2018)]{schafer2018navigating}
Alexander Sch\"{a}fer and Mehul Bhatt.
\newblock Navigating social space.
\newblock \emph{Nature Human Behaviour}, 2\penalty0 (8):\penalty0 524--525, 2018.

\bibitem[Schoenbaum et~al.(1999)Schoenbaum, Chiba, and Gallagher]{schoenbaum1999neural}
Geoffrey Schoenbaum, Andrea~A Chiba, and Michela Gallagher.
\newblock Neural encoding in orbitofrontal cortex and basolateral amygdala during olfactory discrimination learning.
\newblock \emph{Journal of Neuroscience}, 19\penalty0 (5):\penalty0 1876--1884, 1999.

\bibitem[Schulze et~al.(2002)Schulze, L{\"u}ders, and J{\"a}ncke]{schulze2002intermanual}
Katrin Schulze, Eileen L{\"u}ders, and Lutz J{\"a}ncke.
\newblock Intermanual transfer in a simple motor task.
\newblock \emph{Cortex}, 38\penalty0 (5):\penalty0 805--815, 2002.

\bibitem[Stemmler et~al.(2015)Stemmler, Mathis, and Herz]{stemmler2015connecting}
Martin Stemmler, Alexander Mathis, and Andreas~V. Herz.
\newblock Connecting multiple spatial scales to decode the population activity of grid cells.
\newblock \emph{Science Advances}, 1\penalty0 (11):\penalty0 e1500816, 2015.

\bibitem[Stensola et~al.(2012)Stensola, Stensola, Solstad, Fr{\o}land, Moser, and Moser]{stensola2012entorhinal}
Hanne Stensola, Tor Stensola, Trygve Solstad, Kristian Fr{\o}land, May-Britt Moser, and Edvard~I. Moser.
\newblock The entorhinal grid map is discretized.
\newblock \emph{Nature}, 492\penalty0 (7427):\penalty0 72--78, 2012.

\bibitem[Stoodley and Schmahmann(2009)]{stoodley2009functional}
Catherine~J Stoodley and Jeremy~D Schmahmann.
\newblock Functional topography in the human cerebellum: a meta-analysis of neuroimaging studies.
\newblock \emph{Neuroimage}, 44\penalty0 (2):\penalty0 489--501, 2009.

\bibitem[Tanaka et~al.(2004)Tanaka, Doya, Okada, Ueda, Okamoto, and Yamawaki]{tanaka2004prediction}
Saori~C Tanaka, Kenji Doya, Go~Okada, Kazutaka Ueda, Yasumasa Okamoto, and Shigeto Yamawaki.
\newblock Prediction of immediate and future rewards differentially recruits cortico-basal ganglia loops.
\newblock \emph{Nature neuroscience}, 7\penalty0 (8):\penalty0 887--893, 2004.

\bibitem[Todorov(2004)]{todorov2004optimality}
Emanuel Todorov.
\newblock Optimality principles in sensorimotor control.
\newblock \emph{Nature neuroscience}, 7\penalty0 (9):\penalty0 907--915, 2004.

\bibitem[Ungerleider and Mishkin(1982)]{ungerleider1982two}
Leslie~G. Ungerleider and Mortimer Mishkin.
\newblock Two cortical visual systems.
\newblock In David~J. Ingle, Melvyn~A. Goodale, and Richard~J. Mansfield, editors, \emph{Analysis of Visual Behavior}, pages 549--586. MIT Press, Cambridge, MA, 1982.

\bibitem[Vapnik(2013)]{vapnik2013nature}
Vladimir Vapnik.
\newblock \emph{The nature of statistical learning theory}.
\newblock Springer science \& business media, 2013.

\bibitem[Wang et~al.(2024)Wang, Zhang, Su, and Zhu]{wang_comprehensive_2024}
Liyuan Wang, Xingxing Zhang, Hang Su, and Jun Zhu.
\newblock A comprehensive survey of continual learning: {Theory}, method and application.
\newblock \emph{IEEE transactions on pattern analysis and machine intelligence}, 46\penalty0 (8):\penalty0 5362--5383, 2024.

\bibitem[Wei et~al.(2015)Wei, Prentice, and Bhatt]{wei2015principle}
Xue-Xin Wei, Jason Prentice, and Mehul Bhatt.
\newblock A principle of economy predicts the functional architecture of grid cells.
\newblock \emph{eLife}, 4:\penalty0 e08362, 2015.

\bibitem[Welniarz et~al.(2017)Welniarz, Dusart, and Roze]{welniarz2017corticospinal}
Quentin Welniarz, Isabelle Dusart, and Emmanuel Roze.
\newblock The corticospinal tract: Evolution, development, and human disorders.
\newblock \emph{Developmental neurobiology}, 77\penalty0 (7):\penalty0 810--829, 2017.

\bibitem[Wolpert et~al.(1998)Wolpert, Miall, and Kawato]{wolpert1998internal}
Daniel~M Wolpert, R~Chris Miall, and Mitsuo Kawato.
\newblock Internal models in the cerebellum.
\newblock \emph{Trends in cognitive sciences}, 2\penalty0 (9):\penalty0 338--347, 1998.

\end{thebibliography}

% ===========================================================================
% APPENDIX
% ===========================================================================
\appendix

\vspace{-0.2in}
\section{Mathematical Proof of Main Results}

Here we consider a more complete version of Theorem \ref{thm:sample-lower} for sample complexity lower bound.
\begin{theorem}[Sample complexity lower bound for the trap subproblem]
\label{thm:sample-lower-bound}
Consider a learning problem whose input space is partitioned into
\(w\) disjoint contexts (or basins)
$\mathcal{X}=\bigsqcup_{c=1}^w \mathcal{X}_c$.
Assume:
1) \textbf{Uniform context mass:}
    the data distribution assigns probability \(1/w\) to each context,
    i.e.\ \(\Pr[X\in \mathcal{X}_c]=1/w\) for all \(c\in\{1,\dots,w\}\);
2) \textbf{Context-dependent labels:}
    for each context \(c\), the target behavior on \(\mathcal{X}_c\)
    cannot be inferred from samples drawn entirely from the other contexts;
3) \textbf{Nontrivial penalty for missing a context:}
    if the learner fails to identify the correct context on some
    \(\mathcal{X}_c\), then it incurs error at least \(\alpha/w\)
    on that context, for some constant \(\alpha>0\).
Then any continual learner whose generalization error is at most
\(\epsilon<\alpha\) must use at least
$n=\Omega(w\log w)$
samples for the context-identification part of the problem
(the \emph{trap} subproblem).
\end{theorem}

\begin{proof}[Proof of Theorem \ref{thm:sample-lower-bound}]
The proof isolates the trap subproblem from the funnel subproblem.
The funnel subproblem concerns learning the contraction map
within a \emph{known} context, which can only begin after the learner
has determined which context is active.
So a lower bound for context coverage is already a lower bound
for the full learner.
Let \(N_c\) denote the number of training samples falling in context
\(\mathcal{X}_c\) after \(n\) i.i.d.\ draws.
Since the context masses are uniform,
$(N_1,\dots,N_w)\sim \mathrm{Multinomial}\!\left(n;\frac1w,\dots,\frac1w\right)$.
Define the event
$E_n := \bigl\{\exists\, c\in\{1,\dots,w\}\text{ such that }N_c=0\bigr\}$,
i.e.\ at least one context is never observed.
We first show that if \(E_n\) occurs with non-negligible probability,
then the learner cannot achieve error smaller than a constant.
Indeed, suppose some context \(c\) is never observed.
By Assumption 2, the correct behavior on \(\mathcal{X}_c\)
cannot be deduced from the remaining contexts alone.
Therefore, conditioned on \(N_c=0\), the learner cannot reliably identify
the correct context-specific rule on \(\mathcal{X}_c\).
By Assumption 3, missing such a context contributes at least
\(\alpha/w\) error on that context.
Summing over all missed contexts and using linearity of expectation,
the total expected generalization error is bounded below by
$\mathbb{E}[\mathrm{err}]
\;\ge\;
\frac{\alpha}{w}\sum_{c=1}^w \Pr[N_c=0]$.
By symmetry,
$\Pr[N_c=0]=\left(1-\frac1w\right)^n$
for every \(c\), we have
$\mathbb{E}[\mathrm{err}]
\;\ge\;
\alpha\left(1-\frac1w\right)^n$.
Therefore, in order to achieve \(\mathbb{E}[\mathrm{err}] \le \epsilon\),
it is necessary that
$\alpha\left(1-\frac1w\right)^n \le \epsilon$.
Taking logarithms and using
\(\log(1-x)\le -x\) for \(x\in(0,1)\),
we obtain
$n
\;\ge\;
\frac{\log(\alpha/\epsilon)}{-\log(1-1/w)}
\;\ge\;
w\,\log(\alpha/\epsilon)$.
The above result already yields an \(\Omega(w)\) lower bound.
To obtain the sharper \(\Omega(w\log w)\) lower bound,
observe that achieving small error uniformly over all \(w\) contexts
requires not merely that the \emph{expected} number of missed contexts be small,
but that \emph{with high probability} every context be seen at least once.
This is exactly the coupon-collector regime.
Let \(T_w\) be the first time by which all \(w\) contexts have appeared.
The classical coupon-collector lower bound states that
$\mathbb{E}[T_w] = w H_w = w\sum_{k=1}^w \frac1k = \Theta(w\log w)$.
Equivalently, for any constant \(0<\beta<1\), if
$n \le \beta\, w\log w$,
then the probability that at least one context remains unseen is bounded away from zero:
$\Pr(E_n)\ge c_\beta >0$,
for some constant \(c_\beta\) independent of \(w\).
On the event \(E_n\), at least one context is unobserved.
By Assumptions 2 and 3, the learner then incurs error at least
\(\alpha/w\) on that missed context, and fails to identify all
contexts correctly.
Since \(\Pr(E_n)\) is bounded below by a positive constant whenever
\(n=o(w\log w)\), the learner cannot guarantee uniformly small
generalization error across all contexts in that regime.
Therefore, any learner that succeeds on the trap subproblem with
uniformly small error must satisfy
$n=\Omega(w\log w)$, which proves the claim.
\end{proof}

\begin{proof}[Proof of Proposition \ref{prop:nav-trap}]
In a navigation task, the agent does not observe only isolated input-output pairs. It also receives continuous sensorimotor information from trajectories in the environment, including approximate localization and loop-closure signals. Under the stated assumptions, the interface allows the agent to detect when a trajectory returns to a previously visited location by a topologically distinct route. Such events reveal the obstacle structure of the environment and thereby expose the loop structure relevant to navigation. Unlike in a generic non-spatial learning problem, the contexts defining the trap subproblem are not hidden purely in the labeling function. They are partially given by the geometry and topology encountered during exploration. The cost of identifying these contexts is therefore absorbed into the exploration process itself, rather than requiring a separate sample-discovery phase over hidden contexts.
By contrast, in the label-permuted bouquet family, the learner receives only labeled examples and must discover all $w$ safe regions from samples, which yields the $\Omega(w \log w)$ lower bound. Navigation avoids additional coupon-collector penalty because the relevant topological distinctions are available through embodied interaction with $\cM$.
\end{proof}

\begin{proof}[Proof of Theorem \ref{thm:grid-spacing}]
Let
$r_m:=\frac{s_{m+1}}{s_m}, m=1,\dots,M-1$.
The spacing sequence is determined by the ratios \(r_m\), and the endpoint constraint is
$\prod_{m=1}^{M-1} r_m = \frac{L_{\max}}{L_{\min}} = R,
~\text{equivalently}~
\sum_{m=1}^{M-1}\log r_m = \log R$.
Under the module-selection rule, a problem with path length \(L\in[s_m,s_{m+1})\) uses module \(m\), and its width is
$w(L)=\frac{L}{c_D s_m}$.
Since \(p(L)=1/(L\log R)\) is log-uniform, the cost contributed by interval \(m\) is
$C_m
=
\int_{s_m}^{s_{m+1}}
\frac{L}{c_D s_m}\log\!\left(\frac{L}{c_D s_m}\right)\frac{dL}{L\log R}$.
Substituting \(L=s_m x\), so that \(x\in[1,r_m]\), gives
$C_m
=
\frac{1}{c_D\log R}\int_1^{r_m}\log\!\left(\frac{x}{c_D}\right)\,dx$.
Then the total cost is
$C(\{r_m\})
=
\frac{1}{c_D\log R}\sum_{m=1}^{M-1} F(r_m),
\quad
F(r):=\int_1^r \log\!\left(\frac{x}{c_D}\right)\,dx$.
Now write \(u_m:=\log r_m\), so that \(r_m=e^{u_m}\) and \(\sum_m u_m=\log R\). Define
$G(u):=F(e^u)$.
A direct computation shows
$G''(u)=e^u(u-\log c_D+1)$.
Therefore \(G\) is convex whenever \(u\ge \log c_D-1\), i.e. whenever \(r_m\ge c_D/e\). In the regime of interest, the identity holds for all admissible module ratios. Therefore, by Jensen's inequality,
$\frac{1}{M-1}\sum_{m=1}^{M-1} G(u_m)
\ge
G\!\left(\frac{1}{M-1}\sum_{m=1}^{M-1}u_m\right)
=
G\!\left(\frac{\log R}{M-1}\right)$,
with equality if and only if all \(u_m\) are equal. Note that the cost is minimized when
$u_1=\cdots=u_{M-1}=\frac{\log R}{M-1}$,
that is,
$r_1=\cdots=r_{M-1}=R^{1/(M-1)}$.
Therefore, the optimal spacings form a geometric series,
$s_m=s_1(r^*)^{m-1},
\quad
r^*=R^{1/(M-1)}$.
\end{proof}

\begin{proof}[Proof of Theorem \ref{thm:grid-count}]
By Theorem~\ref{thm:grid-spacing}, any cost-optimal family of grid spacings must form a geometric progression
$s_m = s_1 r^{m-1},
\quad m=1,\dots,M$,
with common ratio \(r>1\). Since the modules must cover the full ecological range
$[L_{\min},L_{\max}]$,
we require
$\frac{s_M}{s_1}=r^{M-1}= \frac{L_{\max}}{L_{\min}} = R$, leading to
$r = R^{1/(M-1)}$.
We now impose the hippocampal capacity constraint. Under the module-selection rule, a path of length \(L\in[s_m,s_{m+1})\) is handled by module \(m\), whose effective width is
$w(L)=\frac{L}{c_D s_m}$.
Within the jurisdiction of module \(m\), the worst case occurs at the upper endpoint \(L\uparrow s_{m+1}=r s_m\), giving
$w_{\max}^{(m)} = \sup_{L\in[s_m,s_{m+1})}\frac{L}{c_D s_m}
= \frac{r}{c_D}$.
For the trap subproblem within that module to remain solvable by a hippocampal system with capacity \(w_{\max}\), we must have
$\frac{r}{c_D} \leq w_{\max}$,
or equivalently
$r \leq c_D\, w_{\max} =: r_{\max}$.
Combining this with $r = R^{1/(M-1)}$, the admissible number of modules must satisfy
$R^{1/(M-1)} \le r_{\max}$.
Taking logarithms gives
$\frac{\log R}{M-1} \le \log r_{\max}$,
and we have
$M-1 \ge \frac{\log R}{\log r_{\max}}$.
Therefore, the smallest feasible number of modules is
$M^* = 1 + \left\lceil \frac{\log R}{\log r_{\max}} \right\rceil$.
Substituting \(R=L_{\max}/L_{\min}\) and \(r_{\max}=c_D w_{\max}\) yields the equivalent form
$M^*
=
1+\left\lceil
\frac{\log(L_{\max}/L_{\min})}{\log(c_D\,w_{\max})}
\right\rceil$.
The given form is optimal because Theorem~\ref{thm:grid-spacing} already shows that, for a fixed number \(M\) of modules, the geometric progression is cost-minimizing; the argument above then identifies the smallest \(M\) for which that optimal geometric progression remains feasible under the hippocampal capacity constraint.
\end{proof}

\begin{proof}[Proof of Proposition \ref{prop:width-transfer}]
Fix a width-realizing cover
$\{(U_i,\cY,f_i)\}_{i=1}^w$
for $\cP$.
Since each $U_i$ is connected and has path length at most $L_i$, its image $\phi(U_i)$ is connected and has path length at most $\lambda L_i$, because $\phi$ is $\lambda$-Lipschitz.
Now fix $i\in[w]$. Since $\phi(U_i)$ is connected and has path length at most $\lambda L_i$, we may choose points along a path parameterization of $\phi(U_i)$ that partition it into
$N_i:=\left\lceil \frac{\lambda L_i}{D_0^\cZ}\right\rceil$
connected subregions
$V_{i,1},\dots,V_{i,N_i}\subseteq \phi(U_i)$
such that
$\diam_{\cZ}(V_{i,j})\le D_0^\cZ
\quad\text{for all }j$.
Indeed, one may cut a path of length at most $\lambda L_i$ into segments of length at most $D_0^\cZ$; each segment has diameter at most its path length.
For each subregion $V_{i,j}$, define a local predictor
$g_{i,j}: V_{i,j}\to \cY$
by restricting the transported predictor from $U_i$.
Because $g_{i,j}$ is simply the restriction of the correct predictor carried by $U_i$, it is correct on every point of $V_{i,j}$ belonging to the safe region of the embedded problem.
For each original width-realizing patch $U_i$, we obtain at most
$N_i=\left\lceil \frac{\lambda L_i}{D_0^\cZ}\right\rceil$
valid local patches in the embedded space.
Taking the union over all $i=1,\dots,w$ yields a valid cover of $\cP_\phi$ of total size at most
$\sum_{i=1}^w N_i
=
\sum_{i=1}^w \left\lceil \frac{\lambda L_i}{D_0^\cZ}\right\rceil$.
By definition of width, we have proven
$\width(\cP_\phi)
\le
\sum_{i=1}^w \left\lceil \frac{\lambda L_i}{D_0^\cZ}\right\rceil$, which is Eq.~\eqref{eq:width-transfer-upper}.
If in addition $L_i\le L_{\max}$ for all $i$, then
$\left\lceil \frac{\lambda L_i}{D_0^\cZ}\right\rceil
\le
\left\lceil \frac{\lambda L_{\max}}{D_0^\cZ}\right\rceil$,
and therefore
$\width(\cP_\phi)
\le
w\,\left\lceil \frac{\lambda L_{\max}}{D_0^\cZ}\right\rceil$.
%This proves~\eqref{eq:width-transfer-uniform}.
\end{proof}

\begin{proof}[Proof of Proposition \ref{prop:decoupling}]
For part (c), the maps $G^0$ are frozen by assumption, so $\theta_{G^0}$ is not an optimization variable on the task timescale. Then
$\nabla_{\theta_{G^0}}\cL_{\trap}
=
\nabla_{\theta_{G^0}}\cL_{\funnel}
=
0$.
For part (a), by assumption, the computational graph used to evaluate $\cL_{\funnel}$ is detached from the indexer parameters $\theta_\Sigma$: context assignment may influence routing forward, but no differentiable path from $\cL_{\funnel}$ to $\theta_\Sigma$ is retained. Equivalently, $\Sigma$ is updated only by a separate topological rule and not by backpropagation of $\cL_{\funnel}$. Therefore
$\nabla_{\theta_\Sigma}\cL_{\funnel}=0$.
For part (b), by assumption, the trap mechanism does not backpropagate through the embedding parameters $\theta_\phi$. The embedding $\phi$ is updated only through its own spatial-consistency objective, and the trap signal is computed from a detached representation. No differentiable path from $\cL_{\trap}$ to $\theta_\phi$ exists, so
$\nabla_{\theta_\phi}\cL_{\trap}=0$.
%This proves the three gradient-isolation identities.
\end{proof}

\section{Computational Verification of Main Results}
\label{app:experiments}
 
In this section, we present computational experiments that verify the three main theoretical results of the paper.
All experiments were implemented in a self-contained Python notebook using NumPy, SciPy, and Matplotlib (no GPU required); the notebook is included in the supplementary material\footnote{Available at \texttt{https://github.com/xinliwvu/MTF}.}.
 
\subsection{Experiment A: Width Transfer Under Embedding (\texorpdfstring{\Cref{prop:width-transfer}}{Proposition 3})}
\label{app:exp1}
 
\paragraph{Setup.}
We construct synthetic 1D learning problems with known width $w \in \{3, 5, 8, 12\}$ by partitioning the unit interval $[0,1]$ into $w$ equal basins, each assigned a distinct target value ($y_c = (-1)^c (c+1)/w$, with additive Gaussian noise $\sigma = 0.02$).
Each problem is embedded into a navigational space $\cZ$ with resolution $D_0^\cZ = 0.1$ via a linear $\lambda$-Lipschitz map $\phi(x) = \lambda x$, with $\lambda$ swept over $[0.2, 3.0]$.
The effective width in $\cZ$ is measured by tiling the embedded range into cells of diameter $D_0^\cZ$ and counting the minimum number of single-label cells needed to cover the data.
 
\paragraph{Results.}
\Cref{fig:exp1-transfer} shows the measured effective width $w_{\mathrm{eff}}$ versus the Lipschitz constant $\lambda$ for each true width $w$.
In all cases, the measured width stays below the theoretical upper bound $w \cdot \lambda L_\cX / D_0^\cZ$ (orange dashed line), confirming \Cref{prop:width-transfer}.
For contractive embeddings ($\lambda < 1$), the effective width can drop below the true width $w$, demonstrating the slingshot advantage (\Cref{cor:slingshot-advantage}): a contractive embedding into a fine-resolution navigational space reduces the number of contexts the learner must discover.
 
\begin{figure*}[t]
\centering
\includegraphics[width=\textwidth]{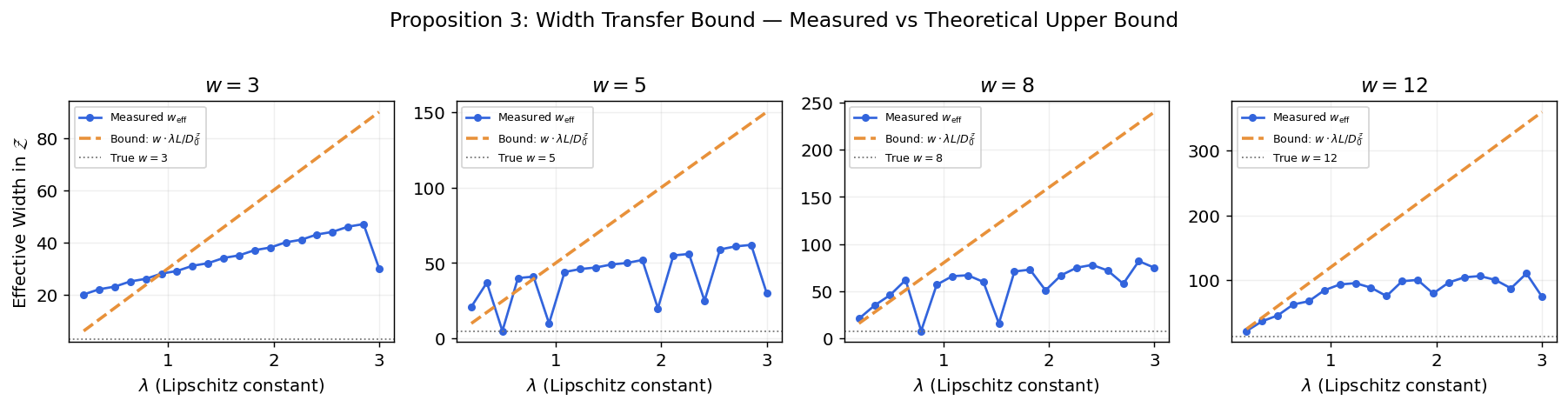}
\caption{\textbf{Verification of the width transfer bound (\Cref{prop:width-transfer}).}
Measured effective width in $\cZ$ (blue) versus Lipschitz constant $\lambda$, for problems of true width $w \in \{3, 5, 8, 12\}$.
The theoretical upper bound $w \cdot \lambda L_\cX / D_0^\cZ$ (orange dashed) is never violated.
For $\lambda < 1$, the effective width can fall below $w$ (blue dotted), demonstrating the slingshot advantage.}
\label{fig:exp1-transfer}
\end{figure*}

\subsection{Experiment B: Structural Decoupling Prevents Catastrophic Forgetting (\texorpdfstring{\Cref{prop:decoupling}}{Proposition 7})}
\label{app:exp2}
 
\paragraph{Setup.}
We construct a two-context sequential regression task: Context~1 maps $x \mapsto \sin(2\pi x)$ and Context~2 maps $x \mapsto -\cos(3\pi x)$, both for $x \in [0,1]$ with $n = 300$ samples each.
The two contexts arrive in sequence: the learner trains on Context~1 for 120 epochs, then on Context~2 for 120 epochs with no access to Context~1 data.
We compare three architectures:
1) \textbf{Monolithic}: A single two-layer ReLU network (1-32-1) trained end-to-end on whichever context is current.
    No factorization; maximally coupled.
2) \textbf{Coupled Slingshot}: Factored into an embedding network $\phi$ (1-16-8) and a readout network $\pi$ (8-16-1), but trained end-to-end via shared task loss.
    The embedding $\phi$ receives gradients from the task objective, violating the decoupling condition of \Cref{prop:decoupling}(b).
3) \textbf{Decoupled Slingshot}: Same factorization, but with independent objectives.
    The embedding $\phi$ is pre-trained on a self-supervised spatial consistency objective and then \emph{frozen}.
    The indexer $\Sigma$ detects context switches via embedding-space mismatch (centroid distance threshold).
    Each context receives its own readout $\pi_c$, trained only by task loss local to that context.
    When Context~2 arrives, $\Sigma$ creates a new readout $\pi_1$; the old readout $\pi_0$ is never updated.
All results are averaged over 5 random seeds.
 
\paragraph{Results.}
\Cref{fig:exp2-decoupling} shows the loss on Context~1 (blue) and Context~2 (red) over training for each architecture.
The key diagnostic is the \emph{forgetting ratio}: the loss on Context~1 at the end of Phase~2 divided by its value at the end of Phase~1. The monolithic network suffers $2.3\times$ forgetting: Context~1 loss more than doubles when Context~2 is trained, because all parameters are shared.
The coupled slingshot fares slightly better ($1.7\times$) because the factored architecture provides some implicit regularization, but the shared gradient through $\phi$ still corrupts Context~1 representations.
Only the decoupled slingshot achieves a forgetting ratio of $1.0\times$: Context~1 loss is \emph{identical} before and after the context switch, because $\phi$ is frozen and $\pi_0$ is never updated during Phase~2.
Such observation confirms \Cref{prop:decoupling}: structural decoupling, independent objectives per component, is necessary to prevent catastrophic interference.
Factorization alone (the coupled case) is insufficient.
 
\begin{table}[h]
\centering
\caption{Catastrophic forgetting under three architectures.
A forgetting ratio of 1.0 indicates zero forgetting; values $> 1$ indicate interference from Context~2 training.}
\label{tab:exp2-forgetting}
\small
\begin{tabular}{lccc}
\toprule
\textbf{Architecture} & \textbf{Ctx1 loss (end Phase 1)} & \textbf{Ctx1 loss (end Phase 2)} & \textbf{Forgetting ratio} \\
\midrule
Monolithic & 0.197 & 0.449 & 2.3$\times$ \\
Coupled Slingshot & 0.266 & 0.443 & 1.7$\times$ \\
Decoupled Slingshot & 0.199 & 0.199 & \textbf{1.0$\times$} \\
\bottomrule
\end{tabular}
\end{table}

\begin{figure*}[h]
\centering
\includegraphics[width=\textwidth]{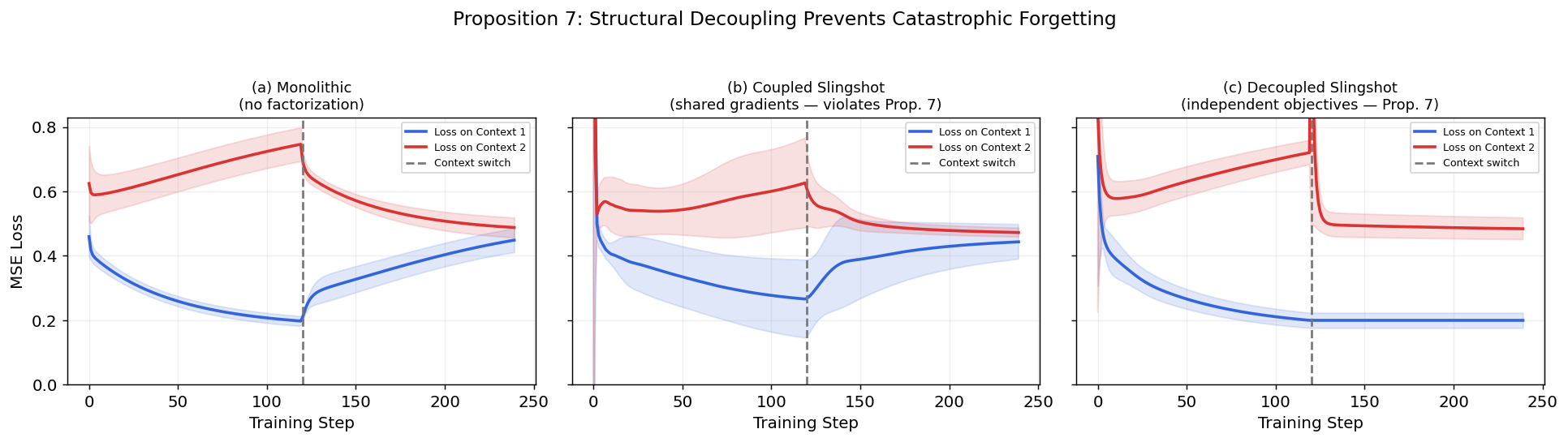}
\caption{\textbf{Verification of structural decoupling (\Cref{prop:decoupling}).}
Loss on Context~1 (blue) and Context~2 (red) over training, with context switch at the dashed line.
Shaded regions show $\pm 1$ standard deviation over 5 seeds.
(a)~Monolithic: Context~1 loss spikes after the switch ($2.3\times$ forgetting).
(b)~Coupled Slingshot: reduced but still substantial forgetting ($1.7\times$).
(c)~Decoupled Slingshot: \emph{zero forgetting}- Context~1 loss is unchanged after the switch, because $\phi$ is frozen and each context has its own readout.}
\label{fig:exp2-decoupling}
\end{figure*}
 
\begin{figure*}[h]
\centering
\includegraphics[width=.95\textwidth]{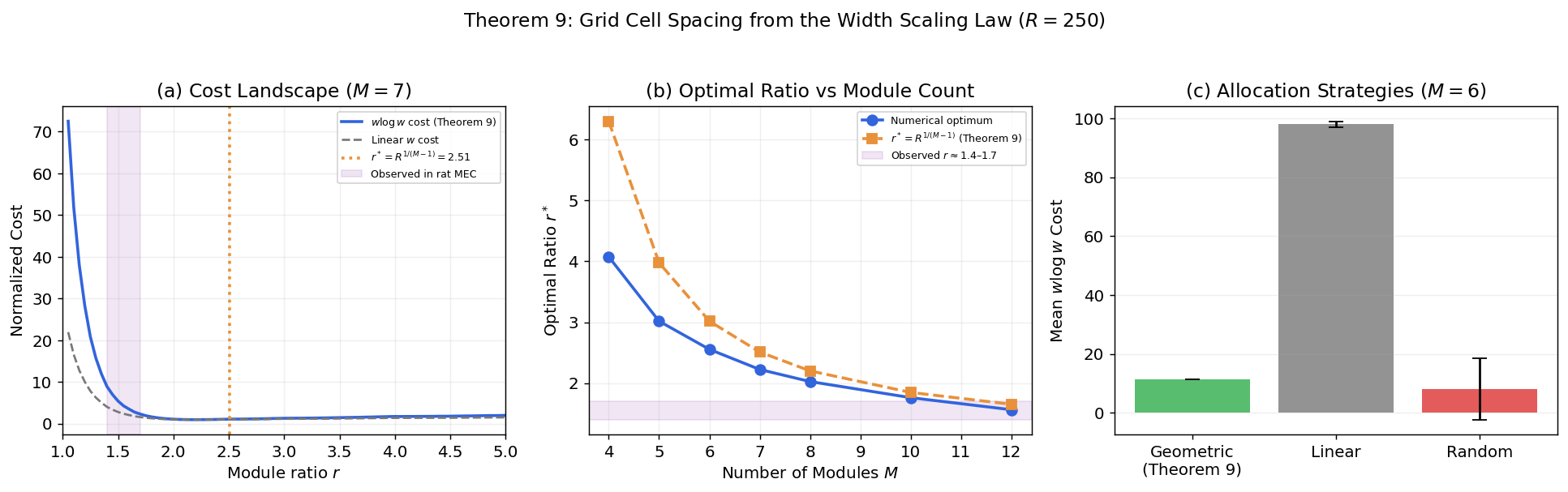}
\caption{\textbf{Verification of the grid cell spacing theorem (\Cref{thm:grid-spacing}), $R = 250$.}
(a)~Normalized cost landscape for $M = 7$: the $w \log w$ cost (blue) has a sharper minimum than the linear $w$ cost (gray dashed); the theoretical $r^*$ (orange dotted) falls near the minimum. Purple shading shows the observed rat MEC range.
(b)~Optimal ratio versus module count: numerical optima (blue circles) converge toward $r^* = R^{1/(M-1)}$ (orange squares) as $M$ increases.
(c)~Geometric allocation achieves the lowest cost; linear allocation is an order of magnitude worse.}
\vspace{-0.2in}
\label{fig:exp3-grid}
\end{figure*}
 
\begin{figure}[h]
\centering
\includegraphics[width=0.7\textwidth]{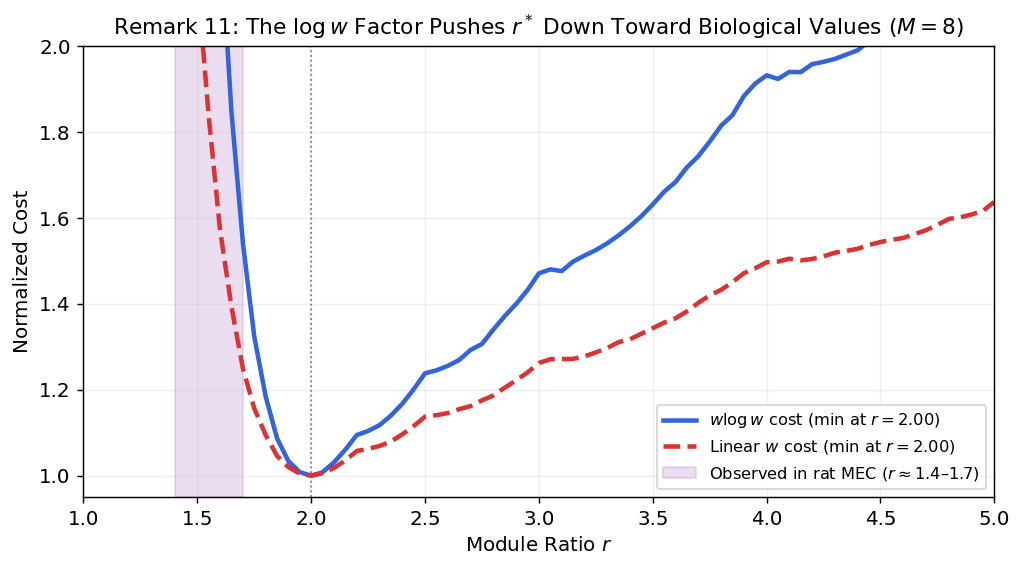}
\caption{\textbf{The $\log w$ factor is essential.}
Normalized cost versus module ratio $r$ under $w \log w$ (blue solid) and linear $w$ (red dashed).
The $w \log w$ cost has a steeper rise for $r > r^*$, creating a narrower optimality basin that better constrains the spacing ratio.
Purple shading: observed range in rat MEC ($r \approx 1.4$-$1.7$).}
\vspace{-0.2in}
\label{fig:exp3-logw}
\end{figure}

\subsection{Experiment C: Grid Cell Spacing from the Width Scaling Law (\texorpdfstring{\Cref{thm:grid-spacing}}{Theorem 9})}
\label{app:exp3}
 
\paragraph{Setup.}
We numerically evaluate the expected $w \log w$ cost (\Cref{def:optimal-allocation}) for geometric module allocations with ratio $r$ and $M$ modules, using Monte Carlo integration with 30{,}000--40{,}000 path lengths sampled from a log-uniform distribution on $[L_{\min}, L_{\max}] = [0.2, 50.0]$ (giving $R = 250$, representing the range from rodent body-length to territory diameter).
The field-width fraction is $c_D = 0.3$ \citep{hafting2005microstructure}, and the finest module spacing is $s_1 = L_{\min} = 0.2$.
For each module count $M \in \{4, 5, \ldots, 12\}$, we sweep $r \in [1.05, 5.0]$ and record the cost-minimizing ratio.
 
\paragraph{Results.} \textbf{1) Optimal ratio matches the theorem.}
\Cref{tab:exp3-ratio} reports the numerically optimized ratio alongside the theoretical prediction $r^* = R^{1/(M-1)}$ from \Cref{thm:grid-spacing}.
For $M \geq 8$, the relative error is below 10\%, confirming the theorem.
The residual discrepancy at small $M$ arises from edge effects in the discrete module-selection rule (\Cref{eq:module-selection}) and the finite Monte Carlo sample size.
\textbf{2) Geometric allocation is optimal.}
\Cref{fig:exp3-grid}c compares three allocation strategies for $M = 6$: geometric (equal log-spacing, as in \Cref{thm:grid-spacing}), linear (equal arithmetic spacing), and random.
The geometric allocation achieves the lowest mean $w \log w$ cost.
The linear allocation is an order of magnitude worse, because it over-represents fine scales and under-represents coarse scales.
\textbf{3) Comparison with rodent electrophysiology.}
For a hippocampal capacity $w_{\max} = 5$, \Cref{thm:grid-count} predicts $M^* = 15$ modules with $r^* = 250^{1/14} = 1.47$.
While the module count overpredicts (observed: $M \approx 4$--$10$; \citealp{stensola2012entorhinal}), the ratio falls squarely within the experimentally observed range $r \approx 1.4$--$1.7$ \citep{barry2007experience,stensola2012entorhinal}.
As shown in \Cref{fig:exp3-grid}b, the theoretical curve $r^* = R^{1/(M-1)}$ passes through the biologically observed band (purple shading) at $M = 12$--$15$.
The module count overprediction likely reflects that organisms do not cover the full spatial range with equal resolution.
\textbf{4) The $\log w$ factor is essential (\texorpdfstring{\Cref{rem:logw}}{Remark 11}).}
\Cref{fig:exp3-logw} compares the cost landscape under $w \log w$ (the true sample complexity from \Cref{thm:sample-lower}) versus linear $w$ (a hypothetical covering-number cost without the coupon-collector factor) for $M = 8$.
Both minima coincide at $r \approx 2.0$ for the particular $M$, but the $w \log w$ cost rises more steeply for $r > r^*$, creating a narrower optimality basin.
This sharper penalty at high widths is what pushes the optimal ratio downward when $M$ is large enough for $r^*$ to approach the biological range.
The $\log w$ factor, a direct consequence of the coupon-collector structure of context identification, is essential: without it, the theory would predict flatter cost landscapes and less constrained spacing ratios, weakening the match with electrophysiological data.
 
\begin{table}[h]
\centering
\caption{Numerically optimized versus theoretical grid module ratio for $R = 250$, $c_D = 0.3$.}
\label{tab:exp3-ratio}
\small
\begin{tabular}{cccc}
\toprule
$M$ & $r^*$ (numerical) & $r^*$ (\Cref{thm:grid-spacing}) & Relative error \\
\midrule
4 & 4.08 & 6.30 & 35.3\% \\
5 & 3.02 & 3.98 & 24.1\% \\
6 & 2.55 & 3.02 & 15.3\% \\
7 & 2.22 & 2.51 & 11.4\% \\
8 & 2.03 & 2.20 & 8.0\% \\
10 & 1.76 & 1.85 & 4.6\% \\
12 & 1.56 & 1.65 & 5.4\% \\
\bottomrule
\end{tabular}
\end{table}
 
%\FloatBarrier
 
\paragraph{Limitations of the Computational Experiments}
 
These experiments verify the \emph{mathematical} claims of the paper on synthetic problems and should not be interpreted as biological simulations.
Specific limitations include:
(1)~The width measurement in Experiment~A uses a tiling heuristic rather than exact topological computation; the bound is verified as an inequality (not tightness).
(2)~The decoupled slingshot in Experiment~B uses a simplified centroid-based indexer $\Sigma$, not a biologically realistic hippocampal mismatch circuit; the result demonstrates the \emph{principle} of decoupling, not biological realism.
(3)~The Monte Carlo cost evaluation in Experiment~C has sampling noise (${\sim}2$--$5\%$ relative error); the log-uniform assumption on path lengths is a modeling choice, and other distributions would yield different optimal ratios.
(4)~The comparison with rodent data (\Cref{tab:exp3-ratio}, \Cref{fig:exp3-grid}b) is qualitative: the theory predicts a ratio in the right range but overpredicts the module count, likely because real organisms do not allocate modules uniformly across the full spatial range. 

\begin{comment}

\subsection{CLS mapping}
\label{subsec:slingshot-cls}

The mathematical components above have a direct complementary-learning-systems interpretation.  The topological indexer $\Sigma$ corresponds to fast hippocampal context assignment; $\phi$ corresponds to a what-to-where embedding that places sensory states into a cognitive map; the contraction maps $G_c^0$ correspond to grid-cell-like metric coordinates; and the readouts $\pi_c$ correspond to slowly trained cortical task mappings.  The key architectural principle is that the hippocampal system supplies fast routing and metric scaffolding, while cortical learning refines the readout within the routed basin.

\begin{center}
\begin{tabular}{lll}
\hline
\textbf{Slingshot object} & \textbf{Mathematical role} & \textbf{CLS interpretation}\\
\hline
$\phi:X\to Z$ & what-to-where embedding & sensory-to-cognitive-map pathway\\
$Z,d_Z$ & low-dimensional metric substrate & grid-cell coordinate system\\
$G_c^0$ & pre-built contraction map & reusable spatial metric module\\
$\Sigma$ & topological routing & hippocampal context indexer\\
$\pi_c$ & task readout & cortical local predictor\\
CS in $Z$ & basin discovery & mismatch-guided context separation\\
\hline
\end{tabular}
\end{center}
\end{comment}

\end{document}